\documentclass[a4paper,10pt]{amsart}

\usepackage{amsaddr}
\usepackage{cite}
\usepackage{amsthm,amssymb}
\usepackage{mathtools,thmtools,thm-restate}
\usepackage{bm,esint}
\usepackage{color}
\usepackage{float,graphicx,subcaption}
\usepackage{cancel}
\usepackage[shortlabels]{enumitem}
\usepackage{mathrsfs}  
\usepackage{bbm}
\usepackage{euscript}
\usepackage{hyperref}
\usepackage{cleveref}
\usepackage{upgreek}
\usepackage{tikz-cd}
\usetikzlibrary{decorations.pathmorphing}

\usepackage{todonotes}

\usepackage{algorithm}
\usepackage{algpseudocode}

\declaretheoremstyle[
  headfont=\normalfont\bfseries\itshape,
  numbered=yes,
  bodyfont=\normalfont,
  spaceabove=1em plus 0.75em minus 0.25em,
  spacebelow=.5em,
  qed={\small$\Diamond$},
]{deflt}

\declaretheorem[style=deflt,numberwithin=section]{theorem}
\declaretheorem[style=deflt,sibling=theorem]{lemma}
\declaretheorem[style=deflt,sibling=theorem]{proposition}
\declaretheorem[style=deflt,sibling=theorem]{corollary}
\declaretheorem[style=deflt,sibling=theorem]{setting}\declaretheorem[style=deflt,sibling=theorem]{assumption}

\declaretheorem[style=deflt,sibling=theorem]{example}
\declaretheorem[style=deflt,sibling=theorem]{remark}
\declaretheorem[style=deflt,sibling=theorem]{definition}

\numberwithin{equation}{section}
\newcommand{\V}{\mathsf{\bm{V}}}
\newcommand{\A}{\mathsf{\bm{A}}}
\newcommand{\RR}{\mathsf{\bm{R}}}
\newcommand{\MM}{\mathsf{\bm{M}}}
\newcommand{\EE}{\mathsf{\bm{E}}}

\newcommand{\T}{\mathbb{T}}
\newcommand{\R}{\mathbb{R}}
\newcommand{\C}{\mathbb{C}}
\newcommand{\E}{\mathbb{E}}
\newcommand{\N}{\mathbb{N}}

\newcommand{\Z}{\mathbb{Z}}

\newcommand{\cD}{\mathcal{D}}
\newcommand{\cE}{\mathcal{E}}
\newcommand{\cF}{\mathcal{F}}
\newcommand{\cG}{\mathcal{G}}
\newcommand{\cH}{\mathcal{H}}

\newcommand{\cK}{\mathcal{K}}
\newcommand{\cL}{\mathcal{L}}
\newcommand{\cM}{\mathcal{M}}
\newcommand{\cN}{\mathcal{N}}

\newcommand{\cP}{\mathcal{P}}
\newcommand{\cQ}{\mathcal{Q}}

\newcommand{\cU}{\mathcal{U}}
\newcommand{\cV}{\mathcal{V}}

\newcommand{\cX}{\mathcal{X}}
\newcommand{\cY}{\mathcal{Y}}

\newcommand{\tQ}{\tilde{\cQ}}

\newcommand{\fD}{\mathfrak{D}}
\newcommand{\fE}{\mathfrak{E}}

\newcommand{\dc}{{d_c}}
\newcommand{\Enc}{\mathrm{Enc}}
\newcommand{\FNO}{\mathrm{FNO}}

\newcommand{\define}{\textbf}

\newcommand{\vol}{\mathrm{vol}}

\renewcommand{\tilde}{\widetilde}
\renewcommand{\hat}{\widehat}

\DeclareMathOperator*{\argmin}{argmin}
\newcommand{\embeds}{{\hookrightarrow}}

\newcommand{\set}[2]{{\left\{ #1 \,\middle|\, #2 \right\}}}
\newcommand{\slot}{{\,\cdot\,}}
\newcommand{\Lip}{\mathrm{Lip}}

\newcommand{\explain}[2]{\overset{\mathclap{\underset{\downarrow}{#2}}}{#1}}

\newcommand{\Err}{\mathrm{Err}}

\tikzcdset{every label/.append style = {font = \normalsize},
            cells={nodes={font=\large}}}

\title[Operator learning: Information-theory]{Operator Learning of Lipschitz Operators: \\
An Information-Theoretic Perspective}
\author{Samuel Lanthaler}
\address{California Institute of Technology}
\email{slanth@caltech.edu}

\date{ \today }

\makeatletter
\renewcommand{\paragraph}{%
  \@startsection{paragraph}{4}%
  {\z@}{.8ex \@plus 1ex \@minus .2ex}{-1em}%
  {\normalfont\normalsize\bfseries}%
}
\makeatother

\usepackage{color}
\definecolor{darkred}{rgb}{.6,0,0}
\definecolor{darkblue}{rgb}{0,0,.7}
\definecolor{darkgreen}{rgb}{0,.7,0}
\definecolor{darkbrown}{rgb}{0.8,0.4,0.4}

\newtheoremstyle{named}{}{}{\itshape}{}{\bfseries}{.}{.5em}{\thmnote{#3}}
\theoremstyle{named}

\begin{document}

\begin{abstract}
Operator learning based on neural operators has emerged as a promising paradigm for the data-driven approximation of operators, mapping between infinite-dimensional Banach spaces. Despite significant empirical progress, our theoretical understanding regarding the efficiency of these approximations remains incomplete. This work addresses the parametric complexity of neural operator approximations for the general class of Lipschitz continuous operators. Motivated by recent findings on the limitations of specific architectures, termed \emph{curse of parametric complexity}, we here adopt an information-theoretic perspective. Our main contribution establishes lower bounds on the metric entropy of Lipschitz operators in two approximation settings; uniform approximation over a compact set of input functions, and approximation in expectation, with input functions drawn from a probability measure. It is shown that these entropy bounds imply that, regardless of the activation function used, neural operator architectures attaining an approximation accuracy $\epsilon$ must have a size that is exponentially large in $\epsilon^{-1}$. The size of architectures is here measured by counting the number of encoded bits necessary to store the given model in computational memory. The results of this work elucidate fundamental trade-offs and limitations in operator learning.
\end{abstract}

\maketitle

\section{Introduction}
\label{sec:I}

Operators mapping between infinite-dimensional Banach spaces of functions are ubiquitous in the natural sciences and engineering. They often appear in connection with physical models expressed as a set of partial differential equations, where operators of interest frequently arise from associated forward and inverse problems, e.g. mapping initial data to the solution at a later time, or identifying external forcing terms from (partial) knowledge of the solution.

Operator learning has emerged as a new paradigm for the data-driven approximation of such operators. Popular operator learning frameworks build on the success of neural networks, but generalize this notion to the infinite-dimensional context of operator approximation, resulting in so-called \emph{neural operators}. These neural operator architectures define parametric mappings, whose parameters are tuned to approximate an underlying operator of interest.

While there is a very rapidly growing body of empirical work demonstrating the great potential, and practical utility, of such data-driven approaches, many open questions remain in our understanding of the theoretical underpinnings of this field, see e.g. \cite{kls2024hna} for a recent review and references therein.

First theoretical insights into specific architectures, and their underlying approximation mechanisms, can be gained by studying universal approximation, i.e. the ability to approximate very general classes of operators. The study of universal approximation of neural operators dates back at least three decades, to early work on operator networks by Chen and Chen \cite{chen1995universal}. Due to the recent rise in the popularity of operator learning and the introduction of a number of novel state-of-the-art frameworks, this early work has been complemented by a number of papers in recent years, demonstrating similar universal approximation properties for various architectures; e.g. DeepONets \cite{lu2021learning,lanthaler2022error}, PCA-Net \cite{bhattacharya2021model,lanthaler2023operator}, Fourier neural operator \cite{kovachki2021universal} and general neural operators \cite{kovachki2023neural,lanthaler2023nonlocal}, as well as multiple other architectures \cite{zhang2023belnet,jin2022mionet,HUA202321,castro2023kolmogorov,castro2022calder}.

Universal approximation implies that there are no fundamental obstructions to operator learning with a given framework, and usually requires identification of basic approximation mechanisms that can be leveraged by a given architecture. However, to determine whether operator learning can be achieved efficiently, a refined quantitative analysis is required. In such quantitative analysis, one often distinguishes between \emph{parametric complexity}, relating the required model size to the achieved accuracy, and \emph{sample or data complexity}, relating the number of required training samples to the achieved accuracy. The focus of the present work is on parametric complexity. For research relevant to the data complexity of operator learning, we mention, for example, \cite{adcock2023a,adcock2023b,NLM2024data,mhaskar1997neural,benitez2023outofdistributional}. 

A general class of operators for which efficient approximation is possible, in terms of the required number of tunable parameters, are so-called holomorphic operators. Research into the approximation of holomorphic operators goes back to the seminal work of Cohen, DeVore and Schwab \cite{cohen2010convergence,cohen2011analytic}, where it was shown that this class of operators can be efficiently approximated by generalized polynomial expansions. More recently, these results have been extended to neural network and neural operator approximation in a series of works \cite{schwab2019deep,opschoor2022exponential,sz2023,herrmann2022neural,adcock2022efficient,marcati2023exponential}, demonstrating that similar rates can be achieved by neural operators. 

 Other classes of operators for which efficient convergence rates have been derived are operator Barron spaces \cite{korolev2021two} and (operator) reproducing kernel Hilbert spaces (RKHS) \cite{lanthalernelsen2023error,nelsen2021random}. Alternative settings, such as parametric PDEs with low-dimensional latent structure are, for example, explored in \cite{kutyniok2022theoretical,lei2022solving,franco2023approximation}.

Apart from these specific classes of operators, efficient approximation has also been established via a case-by-case analysis for several PDE solution operators \cite{deng2021convergence,kovachki2021universal,lanthaler2022error,lanthaler2023operator,marcati2023exponential,galimberti2022designing}. These results identify a number of individual operators of interest which can be efficiently approximated by certain operator learning frameworks. Despite this progress, a general theory encompassing all these examples has yet to emerge. 

A very general class of operators of interest are Lipschitz operators. Approximation theory of relevance to such a general class of operators has been developed e.g. in \cite{liu2024deep,galimberti2022designing,schwab2023deep,patel2024variationally,kratsios2024mixture}. All of these works aim to bound the number of tunable parameters (model size) in terms of the accuracy that can be achieved. 

The present work will focus on deriving lower complexity bounds for the class of Lipschitz continuous operators $\cG: \cD \to \R$, defined on an infinite-dimensional domain $\cD$ and taking values in $\R$ (nonlinear Lipschitz functionals). Semantically, no distinction will be made between `functional' and `operator', since all lower bounds established for functionals continue to hold when considering operators with infinite-dimensional output spaces -- the latter containing (infinitely many) copies of $\R$.

 In addition to the aforementioned literature on neural operator approximation theory, the present work also takes inspiration from the information-theoretic point of view on neural network approximation theory in a finite-dimensional setting, pioneered in the works \cite{bolcskei2017memory,bolcskei2019optimal,petersen2018optimal,voigtlaender2019approximation,epgb2021}, as well as notions of stable approximation \cite{DeVore_Hanin_Petrova_2021,cohen2022optimal}. In the present work, the underlying ideas will be applied and extended to the infinite-dimensional context of operator learning.

The main motivation for this work are two recent results, established in \cite{lanthaler2023operator} and \cite{schwab2023deep} respectively, both applicable to the general setting of Lipschitz operators. A one-paragraph summary of the results in \cite{lanthaler2023operator} and  \cite{schwab2023deep} is as follows:
\begin{enumerate}
\item[(i)] The first result \cite{lanthaler2023operator} shows that certain neural operator architectures, based on ReLU activations,  suffer from a \emph{curse of parametric complexity}: under certain assumptions on the input functions, there exist Lipschitz continuous operators which can only be approximated to accuracy $\epsilon$, if the number of tunable parameters is exponential in $\epsilon^{-1}$; more precisely, the number of parameters must be at least as large as $C\exp(c\epsilon^{-\gamma})$ with problem-dependent constants $C,c,\gamma>0$. 
\item[(ii)] The second result in \cite{schwab2023deep} shows that, under similar assumptions on the input functions, neural operator architectures based on super-expressive activation functions can approximate general Lipschitz operators to accuracy $\epsilon$, with algebraically bounded parameter count; the number of parameters is upper bounded by $C\epsilon^{-\gamma}$, for problem-dependent $C,\gamma>0$.
\end{enumerate}
While the first result, viewed in isolation, appears to hint at fundamental limitations to the development of operator learning theory on the general class of Lipschitz operators, due to the identified ``curse'', the second result shows rigorously that this curse can be circumvented with a suitable choice of activation.

The aim of the present work is to examine the apparent dichotomy between these two results in detail. To this end, we explore the curse of parametric complexity from an information-theoretic perspective. As a result, we will uncover the fundamental information-theoretic character of the curse of parametric complexity, and identify the relevant trade-offs that are possible when parametric complexity is measured by the number of (real-valued) parameters as in \cite{lanthaler2023operator,schwab2023deep}.

\paragraph{Main contributions} 
This work makes the following main contributions:
\begin{itemize}
\item We propose an information-theoretic perspective of operator learning, based on the relation between bit-encoding and Kolmogorov metric entropy; this provides an alternative to the prevalent analysis in the literature, which has focused on estimating the required number of real-valued parameters.
\item For the model class of Lipschitz operators, we derive lower bounds on the metric entropy in two settings: one pertaining to uniform approximation, the other to approximation in expectation. 
\item These bounds imply, in either setting, that an exponentially large number of encoding bits is required to store the weights of any architecture achieving accuracy $\epsilon$ on the model class. This result holds independently of the activation function that is chosen.
\item We use topological arguments to show that even \emph{generic} operators can only be approximated with exponentially increasing complexity; when applied to FNO this implies that the approximation of a \emph{generic Lipschitz operator}, to accuracy $\epsilon$, requires a number of tunable parameters exponential in $\epsilon^{-1}$.
\end{itemize}

\paragraph{Overview}

The remainder of this paper is organized as follows. In Section \ref{sec:main}, we state the main results of this work, as they pertain to operator learning with neural operator architectures. This section contains the main conceptual contributions of this work and reviews the link between bit-encoding and Kolmogorov entropy. Several technical details are left to Sections \ref{sec:3} and \ref{sec:4}; in Section \ref{sec:3}, we derive lower bounds on the Kolmogorov metric entropy of the set of $1$-Lipschitz operators in both a sup-norm and $L^p$-norm approximation setting. In particular, we show that the metric $\epsilon$-entropy increases exponentially with $\epsilon^{-1}$, implying a general curse of parametric complexity for bit-encoded architectures. This is the first main technical contribution of this work. Approximation rates for \emph{generic} operators are the subject of Section \ref{sec:4}, where we first formulate the operator approximation problem in an abstract Banach space setting, and then use topological arguments to relate approximation rates of generic elements of a model class to the metric entropy of this class. This is the second main technical contribution of this work. Finally, Section \ref{sec:conclusion} contains concluding remarks.

\section{Main Results}
\label{sec:main}

This section contains a summary of the main results of this work, applied to the specific setting of operator learning. Several of these results are based on more general, abstract propositions which are included in subsequent Sections \ref{sec:3} and \ref{sec:4}. To aid readability, we leave most technical details to these latter sections. The aim of this section is instead to explain the main ideas underlying our analysis, and their implications for operator learning.  Recurring notation, to be introduced and discussed in the following, is summarized in Table \ref{tab:notation}.

\begin{table}
\begin{tabular}{c|c}
Notation & Meaning \\ \hline
$\cG: \cD \subset \cX \to \cY$ & Nonlinear operator with domain $\cD$ \\
$\cX$, $\cY$ & (Input/output) Banach spaces \\
$\cK \subset \cX$ & Compact subset of inputs \\
$\cD$ & Operator domain, $\cD = \cK$ or $\cD=\cX$ \\
$\mu \in \cP(\cX)$ & Probability measure on $\cX$ \\
$\Lip_1(\cD)$ & Real-valued $1$-Lipschitz operators, $\cG: \cD \to \R$ \\
$\V$ & Banach space of operators, $\V = C(\cK)$ or $\V = L^p(\mu)$ \\
$\A\subset \V$ & Compact subset of $\V$, e.g. $\A = \Lip_1(\cD)$
\end{tabular}
\caption{Recurring notation and definitions for operator learning.}
\label{tab:notation}
\end{table}

\subsection{Operator approximation by neural operators}

We begin the discussion of our main results by proposing an encoder-decoder point of view on operator learning, where the encoder and decoder are implicitly defined by a given architecture. We then define approximation errors of interest and discuss two common measures to quantify the ``complexity'' of a given architecture. The first counts the number of tunable, real-valued parameters in the architecture. The second goes one step further, and requires specification of a bit-encoding of all parameters, i.e. encoding by a sequence of 0's and 1's. To fix intuition, this bit-encoding can be loosely interpreted as the representation of the parameters on computing hardware. The complexity of a bitwise-encoded architecture is measured by the number of bits required to represent it. As will be explained, this provides a link to fundamental information-theoretic concepts such as the Kolmogorov metric entropy of our model class.

\subsubsection{Approximation theoretic setting}
Assume we are given input and output spaces $\cX$, $\cY$. A neural operator defines a parametrized mapping $\Phi: \cX \times \R^{q} \to \cY$, where $\theta \in \R^q$ are tunable parameters. Specification of $\theta$ defines an operator, $\Phi(\slot;\theta): \cX \to \cY$. In practice, the training of a neural operator results in an optimized parameter choice $\theta_{\cG}$ for given $\cG: \cX \to \cY$ and an approximation $\cG \approx \Phi(\slot;\theta_\cG)$.

\paragraph{Model class $\Lip_1(\cD)$} In the following, we will consider a model class of $1$-Lipschitz operators, restricting attention to the case of real-valued outputs, $\cY = \R$:
\begin{definition}[Model class $\Lip_1$]
Let $(\cD,d)$ be a metric space. We define $\Lip_1(\cD)$ as the set consisting of all $1$-Lipschitz continuous mappings $\cG: \cD \to \R$ with $\Vert \cG \Vert_{\Lip} \le 1$, where we define the $\Vert \slot \Vert_{\Lip}$-norm as follows:
\begin{align}
\left\{
\begin{aligned}
\Vert \cG \Vert_{\Lip} &= \max\Big\{ \textstyle\sup_{u \in \cD} | \cG(u) |, \Lip(\cG)\Big\}, \\
\Lip(\cG) &= \sup_{u\ne v} \frac{|\cG(u) - \cG(v)|}{d(u,v)},
\end{aligned}
\right.
\end{align}
\end{definition}

As described in the introduction, the goal of operator learning is to approximate $\cG: \cD \to \R$ by a neural operator $\Phi: \cD \times \R^q \to \R$. In this work, we aim to relate the approximation accuracy $\epsilon$ to the required model size of $\Phi$. 
We will focus on two settings, where either (i) $\cD = \cK\subset \cX$ is a compact subset of a Banach space and the metric is the sup-norm over $\cK$, or (ii) $\cD = \cX$ is a Banach space and the metric is induced by the $L^p(\mu)$-norm with respect to a probability measure $\mu$ on $\cX$ (cp. Table \ref{tab:0}). 

\paragraph{Approximation spaces and norms}
To measure the approximation accuracy of this approximation task, we have to define a distance between operators. To this end, we will consider a Banach space of operators $\V$, allowing for an embedding $\Lip_1(\cD)\subset \V$. Throughout, we will consider one of the following two settings. In the first setting, we aim to approximate $\cG$ over a compact domain $\cD = \cK \subset \cX$:
\begin{setting}[Uniform approximation]
\label{set:uniform}
If $\cG: \cK \to \R$ is an operator with compact domain $\cK \subset \cX$, we will study its uniform approximation over $\cK$, i.e. we take $\V = C(\cK)$ to be the space of continuous operators, metrized by the sup-norm:
\begin{align}
\Vert \cG \Vert_{C(\cK)} = \sup_{u\in \cK} |\cG(u)|.
\end{align}
\end{setting}

A common special case of this setting is the case where $\cK\subset \cX$ is defined by a smoothness constraint, as illustrated by the following example:
\begin{example}
Let $D \subset \R^d$ be a bounded domain. An example of the setting above is the case of Lipschitz operators $\cG: \cK \subset L^2(D) \to \R$, with 
\[
\cK = \set{u\in H^s(D)}{\Vert u \Vert_{H^s(D)} \le C},
\]
a set defined by a Sobolev smoothness constraint for $s>0$. Here, $\cX = L^2(D)$.
\end{example}

In the second setting, we aim to approximate $\cG$ over the entire Banach space $\cD= \cX$, but with respect to a (Bochner) $L^p(\mu)$-norm:
\begin{setting}[Approximation in expectation]
\label{set:expectation}
If $\cG: \cX \to \R$ is an operator with unbounded domain $\cX$ a separable Banach space, then we will assume that inputs are drawn at random from a probability measure $\mu \in \cP(\cX)$. In this case, we fix $p\in [1,\infty)$ and take $\V= L^p(\mu)$ as the space of $\mu$-measurable operators with finite $p$-th norm. $L^p(\mu)$ is metrized by the Bochner $L^p$-norm,
\begin{align}
\Vert \cG \Vert_{L^p(\mu)} = \E_{u\sim \mu}\left[
|\cG(u)|^p
\right]^{1/p}.
\end{align}
\end{setting}

\begin{table}
\begin{tabular}{c|c|c|c}
Operator domain & Operator class & Approximation space & Norm \\ \hline
$\cK \subset \cX$ compact & $\cG \in \Lip_1(\cK)$ & $\V = C(\cK)$ & sup-norm \\
$\cX$ Banach & $\cG \in \Lip_1(\cX)$ & $\V = L^p(\mu)$ & $L^p(\mu)$-norm
\end{tabular}
\caption{Operator approximation settings}
\label{tab:0}
\end{table}

\paragraph{Measures of complexity: Counting parameters versus bits}

We will distinguish two ways of measuring the ``complexity'' of neural operator $\Phi(\slot;\theta)$: one based on the number of tunable (real-valued) parameters, the other requiring bit-encoding (or quantization) of the parameters.

A first intuitive notion of complexity is the minimal number of tunable parameters required to reach approximation accuracy $\epsilon$, i.e. the parameter dimension $q$ of a neural operator $\Phi: \cD \times \R^q \to \cY$. As mentioned in the introduction, this point of view has been prevalent in the development of approximation theory for operator learning. As explained previously, depending on the type of activation function that is used, vastly different conclusions can be reached with this definition of complexity. This fact is well-known in the finite-dimensional setting: For example, it has been shown \cite{MAIOROV199981} that there exist smooth, sigmoidal activation functions for which a neural network of \emph{fixed size} can approximate arbitrary continuous function to arbitrary accuracy, i.e. approximation accuracy $\epsilon$ can be reached with a number of parameters $q=O(1)$.

In practical implementations, real-valued parameters can only be digitally represented to finite accuracy. This observation has led a number of authors \cite{bolcskei2017memory,bolcskei2019optimal,petersen2018optimal,voigtlaender2019approximation,epgb2021}, to analyze neural network approximation from a bit-encoding perspective. In this approach, the continuous parameters $\theta \in \R^q$ are replaced by quantized parameters $\theta \in \Theta$, where $\Theta \subset \R^q$ is a finite set. If the number of elements is bounded, say $|\Theta| = 2^B$ for some $B \in \N$, then we can identify $\Theta \simeq \{0,1\}^B$, i.e. each element in the set $\Theta$ is encoded by a string of $B$ bits. Taking this information-theoretic point of view, it is possible to derive (lower) complexity bounds that are independent of the activation function.

\subsection{Encoder-decoder view of neural operators}

Given the discussion of the last paragraph, we now outline an encoder-decoder point of view on neural operators, emphasizing the difference between ``counting parameters'' and ``counting (encoding) bits''.

\paragraph{Counting parameters}
Let $\Phi: \cD \times \R^q \to \R$ be a neural operator architecture. 
To explain our intuition, we temporarily assume the existence of, and fix an optimal parameter choice $\theta_\cG \in \R^q$ for each $\cG \in \Lip_1(\cD)$, so that
\begin{align}
\theta_\cG \in \argmin_{\theta \in \R^q} \Vert \cG - \Phi(\slot;\theta) \Vert_{\V},
\quad \forall \, \cG \in \Lip_1(\cD),
\end{align}
with respect to the relevant norm of interest on the space of operators $\V \supset \Lip_1(\cD)$. 
The corresponding encoder is then given by
\begin{align}
\label{eq:encoder}
\cE: \Lip_1(\cD) \to \R^q, \quad \cG \mapsto \theta_\cG.
\end{align}
The corresponding decoder is
\begin{align}
\label{eq:decoder}
\cD: \R^q \to \V, \quad \theta \mapsto \Phi(\slot;\theta).
\end{align}
In this way, the operator learning architecture $\Phi$ induces a natural encoder/decoder pair on the relevant space of operators, and we are interested in bounds on the encoding error, either for individual $\cG \in \Lip_1(\cD)$, i.e.
\begin{align}
\label{eq:individual}
\Err(\cG;\Phi)_{\V} = 
\inf_{\theta \in \R^q} \Vert \cG - \Phi(\slot; \theta) \Vert_{\V},
\end{align}
or in a minimax sense, i.e.
\begin{align}
\label{eq:minimax}
\Err(\Lip_1(\cD);\Phi)_{\V}  
= 
\sup_{\cG \in \Lip_1(\cD)}
\inf_{\theta \in \R^q} 
\Vert \cG - \Phi(\slot; \theta) \Vert_{\V}.
\end{align}
Given a desired approximation accuracy $\epsilon>0$, either in the sense \eqref{eq:individual} or \eqref{eq:minimax}, one quantity of interest is the required ``complexity'' of any architecture $\Phi$ achieving this accuracy. The above point of view is consistent with estimates on the required number of parameters $q$.

\paragraph{Counting bits}
As discussed before, the number of parameters $q$ is not a suitable measure of complexity when results \emph{independent of the activation} are sought.  
Therefore, we now assume that the parameters $\theta\in \R^q$ are encoded by $B$ bits. This defines a subset $\Theta \subset \R^q$ consisting of $|\Theta| = 2^{B}$ elements. Each $\theta \in \Theta$ is in correspondence with its bit-encoding $[\theta]\in \{0,1\}^B$. Thus, upon associating with any $\cG \in \Lip_1(\cD)$ the optimal $\theta_\cG \in \Theta$, the continuum encoder \eqref{eq:encoder} is now replaced by a bitwise-encoder,
\begin{align}
\label{eq:bit-encoder}
\fE: \Lip_1(\cD) \to \{0,1\}^B, \quad \cG \mapsto [\theta_\cG],
\end{align}
with bitwise-decoder,
\begin{align}
\label{eq:bit-decoder}
\fD: \{0,1\}^B \to \V,\quad [\theta] \mapsto \Phi(\slot;\theta).
\end{align}
The individual and minimax errors, \eqref{eq:individual} and \eqref{eq:minimax}, have the following bit-encoded counterparts,
\begin{align}
\label{eq:bit-individual}
\Err(\cG; \Phi, \Theta)_{\V} = 
\inf_{\theta \in \Theta} \Vert \cG - \Phi(\slot; \theta) \Vert_{\V}.
\end{align}
and
\begin{align}
\label{eq:bit-minimax}
\Err(\Lip_1(\cD);\Phi,\Theta)_{\V} = 
\sup_{\cG \in \Lip_1(\cD)} \inf_{\theta \in \Theta} \Vert \cG - \Phi(\slot; \theta) \Vert_{\V}.
\end{align}

In the present work, we will focus on such a bit-encoding point of view, but mention that there are close links between these two points of view, if the mapping $\theta \mapsto \Phi(\slot;\theta)$ possesses some stability properties. Specifically, this link will be used to derive lower complexity bounds for the Fourier neural operator in Section \ref{sec:FNO}.

\subsection{Information-theoretic notions}

The relevance of the bit-encoding point of view is that it relates directly to the (Kolmogorov) metric entropy of the underlying model class $\A\subset \V$ and allows results to be derived which are independent of specifics of the architecture such as the choice of activation function. Thus bit-encoding enables analysis relating directly to intrinsic topological properties of $\A$. 

\paragraph{Minimax code-length}
Abstracting further our previous discussion, we make the following formal definition of abstract bitwise encoder/decoder pairs:

\begin{definition}[Abstract bitwise encoder/decoder pairs]
Given a compact subset $\A \subset \V$ of a Banach space $\V$, we denote by $\Enc_B(\A;\V)$ the set of all bitwise encoder/decoder pairs $(\fE,\fD)$ of length $B$, i.e. all pairs of mappings $\fE: \A \to \{0,1\}^B$ and $\fD: \{0,1\}^B \to \V$. 
\end{definition}

Following \cite{bolcskei2019optimal}, for $\epsilon > 0$, we also introduce the minimax code length $\cL(\A;\epsilon)_\V$ of a compact set $\A \subset \V$ as the minimal number of bits $B$ for which there exists an (abstract) encoder/decoder pair $(\fE,\fD)\in \Enc_B(\A;\V)$ such that 
\[
\sup_{\cG \in \A} \Vert \cG - \fD \circ \fE(\cG) \Vert_{\V} \le \epsilon.
\]
That is,
\begin{align}
\label{eq:minimax-codelength}
\cL(\A;\epsilon)_\V := \min 
\set{
B \in \N
}{
\begin{gathered}
\exists \, (\fE,\fD) \in \Enc_B(\A;\V) \text{ s.t. } \\
\textstyle\sup_{\cG \in \A} \Vert \cG - \fD \circ \fE(\cG) \Vert_{\V} \le \epsilon 
\end{gathered}
}.
\end{align}

\paragraph{Kolmogorov metric entropy}

Given a metric space $(\V,d)$, element $g\in \V$ and $r>0$, we denote by 
\[
\overline{B_r}(g) := \set{f \in \V}{d(g,f) \le \epsilon},
\]
the \emph{closed} ball of radius $r$. We now make the following definition for the covering number and (Kolmogorov) metric entropy:
\begin{definition}[Covering number and  metric entropy]
\label{def:entropy}
Let $(\V,d)$ be a metric space. For $\epsilon>0$, the $\epsilon$-covering number of a set $\A \subset \V$, denoted $\cN(\A;\epsilon)_\V$, is the smallest integer $N\in \N$, such that $\A$ can be covered by $N$ closed balls of radius $\epsilon$, i.e.
\begin{align}
\label{def:covering}
\cN(\A;\epsilon)_{\V} := \min\set{N\in \N}{\exists \, g_1,\dots, g_N\in \V, \text{ s.t. } \A \subset \textstyle{\bigcup_{j=1}^N} \overline{B_\epsilon}(g_j)}.
\end{align}
We note that the subscript $\V$ is used as a shorthand for $(\V,d)$, with the relevant metric $d$ implied. The metric entropy of $\A\subset \V$ is defined as the logarithm (to base $2$) of the covering number, i.e. 
\begin{align}
\label{eq:entropy}
\cH(\A;\epsilon)_{\V} = 
\log_2 \cN(\A;\epsilon)_{\V}.
\end{align}
\end{definition}

% We recall that the $\epsilon$-covering number of $\A\subset \V$ is the smallest number of balls of radius $\epsilon$, required to cover $\A$:
% \begin{align}
% \cN(\A;\epsilon)_\V 
% :=
% \min \set{N \in \N}{\exists \, u_1,\dots, u_N, \text{ s.t. } \A \subset \bigcup_{j=1}^N B_\epsilon(u_j) }.
% \end{align}
% The metric entropy is defined as the logarithm, to the base 2, of the covering number,
% \begin{align}
% \cH(\A;\epsilon)_\V 
% := \log_2 \cN(\A;\epsilon)_\V.
% \end{align}

\paragraph{Link between minimax code-length and metric entropy}
The minimax code-length and metric entropy introduced in the previous paragraphs are linked by the following fundamental result \cite[Rmk. 5.10]{dahlke2015harmonic}:
\begin{proposition}
\label{prop:code-length-entropy}
Let $\V$ be a Banach space, and let $\A\subset \V$ be compact. Then the metric entropy of $\A$ provides a lower bound on the minimax code length:
\begin{align}
\cL(\A;\epsilon)_\V
\ge \cH(\A;\epsilon)_{\V}.
\end{align}
\end{proposition}

\begin{proof}
Let $\epsilon > 0$ be given. Let $(\fE,\fD)$ be a bitwise encoder/decoder pair with $B= \cL(\A;\epsilon)_\V$ bits, achieving reconstruction error at most $\epsilon$ on $\A$. The image of $\fD: \{0,1\}^B \to \V$ contains at most $N = 2^B$ elements, $\cG_1,\dots, \cG_N$. Since, for any $\cG \in \A$, the specific choice $\fD \circ \fE(\cG)$ belongs to the image of $\fD$, it follows that
\[
\sup_{\cG \in \A}
\inf_{n=1,\dots, N} \Vert \cG - \cG_n \Vert 
\le \sup_{\cG \in \A}  \Vert \cG - \fD \circ \fE(\cG) \Vert \le \epsilon.
\]
Thus, $\A \subset \bigcup_{n=1}^N \overline{B_\epsilon}(\cG_n)$, implying that the covering number of $\A$ is bounded by
\[
\cN(\A;\epsilon)_\V \le N = 2^B.
\]
Taking logarithms and recalling that $B = \cL(\A;\epsilon)_\V$ yields the claim.
\end{proof}

In particular, Proposition \ref{prop:code-length-entropy} implies that if $\cH(\A;\epsilon)_\V > B$, then there \emph{cannot exist} a bit-encoder-decoder pair $(\fE,\fD) \in \Enc_B(\A;\V)$ achieving uniform decoding accuracy $\epsilon$ over $\A$.  Conversely, if $(\fE,\fD)$ is an encoder-decoder pair \eqref{eq:bit-encoder}, \eqref{eq:bit-decoder} associated with a bit-encoded neural operator $\Phi: \cD \times \Theta \to \R$ with $|\Theta| \le 2^B$, and if the following minimax approximation bound holds,
\[
\sup_{\cG \in \Lip_1(\cD)} \inf_{\theta \in \Theta} \Vert \cG - \Phi(\slot;\theta)\Vert_{\V} \le \epsilon,
\]
this implies that $B \ge \cH(\Lip_1(\cD);\epsilon)_{\V}$. 

\subsection{Information-theoretic minimax bounds}

As a consequence of Proposition \ref{prop:code-length-entropy}, we can derive a lower bound on the required number of bits $B$ to achieve the minimax bound \eqref{eq:bit-minimax} by estimating the entropy of $\Lip_1(\cD)\subset \V$. As mentioned before, we will consider two settings, corresponding to uniform approximation of $\cG$ over a compact set $\cK$ (the setting $\cD = \cK$) and approximation with respect to a Bochner $L^p(\mu)$-norm for probability measure $\mu$ (the setting $\cD = \cX$).

\paragraph{Uniform approximation}

We now consider $\cK\subset\cX$ a compact set of input functions, and operators belonging to $\Lip_1(\cK) \subset C(\cK)$ (cp. Setting \ref{set:uniform}). This corresponds to the choice $\cD = \cK$, $\A = \Lip_1(\cK)$, $\V = C(\cK)$, in the discussion of the previous section. We then have the following result:
\begin{theorem}
\label{thm:minimax-uniform}
Let $\cX$ be a Banach space. Let $\cK \subset \cX$ be a compact set of input functions, and assume that the metric entropy of $\cK$ satisfies the lower bound, $\cH(\cK;\epsilon)_{\cX} \ge c_\alpha \epsilon^{-1/\alpha}$ for $\alpha > 0$. There exists a constant $c>0$, independent of $\epsilon$, such that the following holds: If $\Phi: \cK \times \Theta \to \R$ is a quantized neural operator architecture, satisfying 
\[
\sup_{\cG \in \Lip_1(\cK)} \inf_{\theta\in \Theta} \Vert \cG - \Phi(\slot;\theta) \Vert_{C(\cK)} \le \epsilon.
\]
and if $|\Theta|\le 2^B$, i.e. if the parameters of $\Phi$ can be encoded by $B$ bits, then
\[
B \ge \exp(c \epsilon^{-1/\alpha}).
\]
\end{theorem}

\begin{proof}
The claim follows from the relation between the minimax code-length and the metric entropy of $\Lip_1(\cK)\subset C(\cK)$, stated in the above Proposition \ref{prop:code-length-entropy}, and the following general bound on $\cH(\Lip_1(\cK),\epsilon)_{C(\cK)}$:
\[
\cH(\Lip_1(\cK);\epsilon)_{C(\cK)} \ge 2^{\cH(\cK,6\epsilon)_\cX}.
\]
This bound will be shown in Section \ref{sec:sup-norm}, Proposition \ref{prop:sup-norm}. Assuming this bound, then by assumption on $\cK$, we have $2^{\cH(\cK,6\epsilon)_\cX} \ge \exp(c\epsilon^{-1/\alpha})$ for constant $c>0$.
\end{proof}

If $\cX$ is a function space, then compact subsets $\cK\subset \cX$ are commonly defined by a smoothness constraint, and this partly motivates our assumption on $\cK$ in the last theorem. The following example is illustrative. 

\begin{example}
Let $D \subset \R^d$ be a bounded domain. Let $\cX = L^2(D)$. An example of the setting outlined above is the case of Lipschitz operators $\cG: \cK\to \R$, with 
\[
\cK = \set{u\in H^s(D)}{\Vert u \Vert_{H^s(D)} \le C},
\]
defined by a Sobolev smoothness constraint for $C,s>0$. In this case, it is well-known that the metric entropy satisfies $\cH(\cK;\epsilon)_{\cX} \gtrsim \epsilon^{-d/s}$, i.e. the assumptions of Theorem \ref{thm:minimax-uniform} hold with $\alpha = s/d$.
\end{example}

\paragraph{Approximation in expectation}

Another commonly studied setting concerns the approximation in expectation (cp. Setting \ref{set:expectation}). Here, we consider $1$-Lipschitz mappings $\cG: \cX \to \R$ defined on a separable Hilbert space $\cX$. We fix a probability measure $\mu$ on $\cX$ and consider inputs as random draws $u\sim \mu$. 
To derive quantitative lower bounds, we will need to make minimal structural assumptions on $\mu$. 
\begin{assumption}
\label{ass:prob}
There exists an orthonormal basis $e_1,e_2,\dots$ of $\cX$, probability space $(\Omega, \mathbb{P})$ and summable coefficients $\lambda_1 \ge \lambda_2 \ge \dots$, such that $\mu$ is the law of a random variable $u: \Omega \to \cX$ of the form,
\begin{align}
u(\omega) = \sum_{j=1}^\infty \sqrt{\lambda_j} Z_j(\omega) e_j, \quad (\omega \in \Omega).
\end{align}
where $Z_j: \Omega \to \R$ are jointly independent random variables. We assume that the random variable $Z_j$ satisfies $\E |Z_j|^2 = 1$, and has law $Z_j \sim \rho_j(z) \, dz$ for a probability density function $\rho_j: \R \to \R_+$. We furthermore assume that there exists a constant $L > 0$, such that 
\begin{align}
\label{eq:rhobd}
\sup_{j\in \N} \Vert \rho_j \Vert_{L^\infty(\R)} \le L, 
\quad \sqrt{\lambda_1} \le L.
\end{align}
\end{assumption}

A concrete, and widely considered, example satisfying Assumption \ref{ass:prob} is the case of a Gaussian probability measure $\mu$ with prescribed mean and covariance operator. In this case, $\lambda_j$ are the eigenvalues of the covariance operator, $e_j$ the corresponding eigenfunctions, and the random variables $Z_j \sim \rho_j$ have standard Gaussian distribution.

\begin{theorem}
\label{thm:minimax-Lp}
Let $\cX$ be a Banach space of input functions. Let $\mu \in \cP(\cX)$ be a probability measure satisfying Assumption \ref{ass:prob}. Assume that the coefficients $\sqrt{\lambda_j} \gtrsim j^{-\alpha}$ as $j\to \infty$, where $\alpha > 0$. Then there exists a constant $c>0$, independent of $\epsilon$, such that the following holds: If $\Phi: \cX \times \Theta \to \R$ is a quantized neural operator architecture, satisfying 
\[
\sup_{\cG \in \Lip_1(\cK)} \inf_{\theta\in \Theta} \Vert \cG - \Phi(\slot;\theta) \Vert_{L^p(\mu)} \le \epsilon.
\]
and if $|\Theta|\le 2^B$, i.e. if the parameters of $\Phi$ can be encoded by $B$ bits, then
\[
B \ge \exp(c \epsilon^{-1/(\alpha+1)}).
\]
\end{theorem}

\begin{proof}
Similarly to the uniform case, the present claim again follows from the relation between the minimax code-length and the metric entropy of $\Lip_1(\cX)\subset L^p(\mu)$ of Proposition \ref{prop:code-length-entropy}, together with the following general bound on $\cH(\Lip_1(\cX),\epsilon)_{L^p(\mu)}$:
\[
\cH(\Lip_1(\cX);\epsilon)_{L^p(\mu)} \ge \exp(c\epsilon^{-1/(\alpha+1)}).
\]
This lower entropy bound will be derived in Section \ref{sec:Lp-norm}, Proposition \ref{prop:prob}.
\end{proof}

Thus, an exponential number of encoding bits is also needed in an $L^p(\mu)$-setting. Theorem \ref{thm:minimax-Lp} shows that the approximation of Lipschitz operators in expectation is not ``qualitatively'' easier than uniform approximation of such operators over a compact set of input functions.

\subsection{Approximation of generic Lipschitz operators}

Theorems \ref{thm:minimax-uniform} and \ref{thm:minimax-Lp} show that operator learning architectures that can approximate arbitrary $1$-Lipschitz operators to accuracy $\epsilon$ have exponential memory requirements; any (bit-encoded) implementation of such an architecture will require a number of bits that is exponential in $\epsilon^{-1}$. The reason for this is that the space of Lipschitz operators is exponentially large in a fundamental information-theoretic sense quantified by the metric entropy. 

However, this minimax bound applies to the approximation of the \emph{entire class} $\Lip_1(\cD)$ by a single architecture, and does not necessarily imply that it is impossible to approximate \emph{individual} $\cG \in \Lip_1(\cD)$ efficiently. At first sight, it could appear that arguments based on the metric entropy cannot be used to gain any insight into this refined question; Indeed, if we fix individual $\cG\in \Lip_1(\cD)$, then the metric entropy of the singleton-set $\A = \{\cG\}$ is trivially $=0$, and the minimax code length \eqref{eq:minimax-codelength} is $=1$ for any value of the accuracy $\epsilon$, since the trivial decoder $\fD(\slot) \equiv \cG$ reproduces $\cG$ exactly, with vanishing approximation error, $\epsilon = 0$. Thus, while entropy arguments give insights into the (concurrent) approximation of the set $\Lip_1(\cD)$, they seemingly have no immediate implications for the approximation of individual $\cG \in \Lip_1(\cD)$.

Despite these facts, the results below will show that a refined analysis based on the concept of metric entropy is nevertheless possible; in the uniform and $L^p$-settings of the previous section, a fixed sequence of bit-encoded architectures $\{\Phi_n\}_{n\in \N}$, with at most $n$ bits, can approximate \emph{generic} elements $\cG\in \Lip_1(\cD)$ at best at a logarithmic rate, $\Err(\cG;\Phi_n,\Theta_n) \lesssim \log(n)^{-\gamma}$ for fixed $\gamma > 0$. Before stating our result, we briefly recall the notion of a generic element of a (compact) metric space (see Appendix \ref{app:baire} for further remarks, and \cite[Chap. 8]{munkres} for an in-depth discussion):

\begin{definition}[Topologically generic properties]
Let $(\A,d)$ be a compact metric space. A subset $\RR \subset \A$ is called \define{residual}, if it is equal to a countable intersection of sets, each of whose interior is dense in $\A$. The complement of a residual set is a \define{meagre} set. A property $P$ is called \define{generic}, if the set 
\[
\RR := \set{\cG \in \A}{ \cG \text{ satisfies } P} \subset \A,
\]
is residual. 
\end{definition}

Under the assumption that $(\A,d)$ is compact, the Baire category theorem (cp. Appendix \ref{app:baire}) implies that any residual set $\RR$ is dense in $\A$. Furthermore, the intersection $\RR = \bigcap_{j=1}^\infty \RR_j$ of countably many residual sets $\RR_1,\RR_2,\dots$ is itself residual, and hence still dense. In this sense, a topologically generic property is somewhat analogous to a property that holds with probability $1$ in a probabilistic sense. Thus, a generic property is often thought of as a property that is satisfied by ``almost every'' element of $\A$.

We can now state our main results on the approximation of generic operators $\cG \in \Lip_1(\cD)$. In the uniform setting (cp. Setting \ref{set:uniform}), we have:

\begin{proposition}[Uniform approximation of generic operators]
\label{prop:uniform-generic}
Let $\cX$ be a Banach space of input functions. Let $\cK \subset \cX$ be compact, and assume that the metric entropy $\cH(\cK;\epsilon)_{\cX} \gtrsim \epsilon^{-1/\alpha}$ for $\alpha > 0$. Let $\{\Phi_n: \cK \times \Theta_n \to \R\}_{n \in \N}$ be a sequence of bit-encoded neural operator architectures, with quantized parameter set $|\Theta_n| \le 2^n$. Then generic $\cG \in \Lip_1(\cK)$ cannot be approximated by $\{\Phi_n\}$ at a convergence rate better than $\log(n)^{-\alpha}$; more precisely, for any sequence $\epsilon_n = o(\log(n)^{-\alpha})$, there is a residual subset $\RR \subset \Lip_1(\cK)$, consisting of operators $\cG\in \RR$, for which
\[
\inf_{\theta\in \Theta_n} \Vert \cG - \Phi_n(\slot;\theta) \Vert_{C(\cK)} \not = O(\epsilon_n), 
\quad (n\to \infty).
\].
\end{proposition}

\begin{proof}
We let $\V := C(\cK)$ and $\A:= \Lip_1(\cK)$. We note that $\A \subset \V$ is a compact, convex subset. We then consider the sequence of subsets $\Sigma_n \subset C(\cK)$, defined by all possible realizations,
\[
\Sigma_n := \set{\Phi_n(\slot;\theta)}{\theta \in \Theta_n}.
\]
By assumption, $|\Sigma_n| = |\Theta_n| \le 2^n$. By Proposition \ref{prop:sup-norm}, to be proved in Section \ref{sec:sup-norm}, we have $\cH(\A,\epsilon)_\V \ge \exp(c\epsilon^{1/\alpha})$. The claim of Proposition \ref{prop:uniform-generic} then follows, as a special case, from the abstract result of Proposition \ref{prop:abs-exp} to be derived in Section \ref{sec:4}.
\end{proof}

A similar result holds for approximation of Lipschitz operators in an $L^p(\mu)$ sense, as shown in the following proposition  (cp. Setting \ref{set:expectation}):

\begin{proposition}[Approximation of generic operators in expectation]
\label{prop:Lp-generic}
Let $\cX$ be a Banach space of input functions. Let $\mu \in \cP(\cX)$ be a probability measure satisfying Assumption \ref{ass:prob}. Assume that the coefficients $\lambda_j \gtrsim j^{-2\alpha}$ as $j\to \infty$, where $\alpha > 0$. Let $\{\Phi_n: \cX \times \Theta_n \to \R\}_{n \in \N}$ be a sequence of bit-encoded neural operator architectures, with quantized parameter set $|\Theta_n|\le 2^n$. Then generic $\cG \in \Lip_1(\cX)$ cannot be approximated by $\{\Phi_n\}$ at a convergence rate better than $\log(n)^{-(\alpha+1)}$; more precisely, for any sequence $\epsilon_n = o(\log(n)^{-(\alpha+1)})$, there is a residual subset $\RR \subset \Lip_1(\cX)$, such that for any $\cG\in \RR$,
\[
\inf_{\theta\in \Theta_n} \Vert \cG - \Phi_n(\slot;\theta) \Vert_{L^p(\mu)} \not = O(\epsilon_n), 
\quad (n\to \infty).
\].
\end{proposition}

\begin{proof}
We let $\V := L^p(\mu)$ and $\A:= \Lip_1(\cX)$. We note that $\A \subset \V$ is a compact, convex subset. We consider the subsets $\Sigma_n \subset \V$, defined by all possible realizations,
\[
\Sigma_n := \set{\Phi_n(\slot;\theta)}{\theta \in \Theta_n}.
\]
By assumption, $|\Sigma_n| = |\Theta_n| \le 2^n$. By Proposition \ref{prop:prob}, to be proved in Section \ref{sec:Lp-norm}, we have $\cH(\A,\epsilon)_\V \ge \exp(c\epsilon^{-1/(\alpha+1)})$. The claim of Proposition \ref{prop:Lp-generic} then follows, as a special case, from the abstract result of Proposition \ref{prop:abs-exp} to be derived in Section \ref{sec:4}.
\end{proof}

\begin{remark}
The notion of a residual subset $\RR \subset \Lip_1(\cD)$ in Proposition \ref{prop:uniform-generic} and \ref{prop:Lp-generic} is to be understood with respect to the subspace topology on $\Lip_1(\cD)$, induced by the $C(\cK)$ and $L^p(\mu)$-norms, respectively.
\end{remark}

\subsection{Approximation of generic Lipschitz operators by FNO}
\label{sec:FNO}
The results of the previous section are formulated abstractly for an unspecified sequence of quantized neural operator architectures $\{\Phi_n\}$. To conclude the discussion of our main results, we illustrate some implications of these results for a concrete operator learning framework, the Fourier neural operator \cite{li2021fourier}.

We note that although the derivation of these results will rely on Propositions \ref{prop:uniform-generic} and \ref{prop:Lp-generic}, the ultimate statement of the theorems will be in terms of the number of tunable real-valued parameters of FNO, \emph{without bit-encoding}. Thus, the gap between the bit-encoded parameters and real-valued parameters point of view can be bridged in this case.

In preparation to stating these theorems for FNO, we briefly describe a specific setting to which FNO is applicable, and recall the FNO architecture. This is followed by the statement of a novel theorem establishing a curse of (exponential) parametric complexity for the FNO, in the uniform approximation setting.

\paragraph{FNO case study}

As a case study, we consider Fourier neural operators (FNO), approximating a relevant class of $1$-Lipschitz operators,
\[
\cG: \cK \subset L^2(D;\R^{d_{\mathrm{in}}}) \to \R,
\]
mapping square-integrable input functions to the reals (or equivalently, to a space of constant-valued functions). Here $\cK$ is a compact subset of $L^2(D;\R^{d_{\mathrm{in}}})$, consisting of square-integrable functions $u: D \to \R^{d_{\mathrm{in}}}$. We wish to approximate such $1$-Lipschitz operator $\cG$, uniformly over the compact set $\cK$. 

In the following, we will usually write $L^2(D)$ instead of $L^2(D;\R^{\mathrm{in}})$, where for simplicity and due to certain restrictions of the FNO architecture, the underlying domain $D = \T^d$ is taken to be the $1$-periodic torus $\T^d \simeq [0,1]^d$ in $d$ spatial dimensions, where in typical applications, $d\in \{1,2,3\}$. 
Prototpyical examples of relevant $\cK$ are $\cK = \cU(H^s(\T^d))$, where 
\[
\cU(H^s(\T^d)) = \set{u\in H^s(\T^d)}{\Vert u \Vert_{H^s} \le 1},
\]
denotes the unit ball in the Sobolev space $H^s(\T^d)$ with smoothness $s>0$.
The question to be addressed is how many tunable parameters $q$ are needed to approximate generic $\cG \in \Lip_1(\cK)_{L^2(D)}$ to a prescribed accuracy $\epsilon$?

\paragraph{FNO architecture}
We here recall the general notion of Fourier neural operators \cite{li2021fourier}. Let $\cX = \cX(D; \R^{d_{\mathrm{in}}})$ and $\cY = \cY(D; \R^{d_{\mathrm{out}}})$ be two Banach function spaces, consisting of functions $u: D \to \R^{d_\mathrm{in}}$ and $w: D \to \R^{d_{\mathrm{out}}}$, respectively. A Fourier neural operator (FNO) defines a nonlinear operator 
\[
\Phi_{\FNO}: \cX (D; \R^{d_{\mathrm{in}}}) \to \cY(D; \R^{d_{\mathrm{out}}}),
\]
mapping between these spaces.
By definition of the FNO architecture, such $\Phi_{\FNO}$ takes the form 
\begin{align}
\label{eq:FNO}
\Phi_{\FNO}(u;\theta) = Q \circ \cL_L \circ \dots \circ \cL_1 \circ P(u). 
\end{align}
where $P: \cX \to \cV$, $u(x) \mapsto Pu(x)$ is a linear lifting layer, $Q: \cV \to \cY$, $v(x) \mapsto Qv(x)$ is a linear projection layer, and the $\cL_\ell: \cV(D;\R^{\dc}) \to \cV(D;\R^{\dc})$ are the hidden layers, mapping between hidden states $v \mapsto \cL_\ell(v) \in \cV(D;\R^{\dc})$. The hidden states are vector-valued functions with $\dc$ components,  $v: D \to \R^{\dc}$,  belonging to a Banach function space $\cV(D;\R^\dc)$. Here, the ``channel width'' $\dc$ is a hyperparameter of the architecture. Each hidden layer $\cL_\ell$ is of the form
\[
\cL_\ell(v)(x) 
:=
\sigma \big (
Wv(x) + Kv(x) + b
\big )
\]
where $W \in \R^{\dc\times \dc}$ is a matrix multiplying $v(x)$ pointwise, and $b\in \R^\dc$ is a bias. $K$ is a non-local operator of the form
\[
v(x) \mapsto (Kv)(x) := \cF^{-1} \big ( \hat{P}_k \cF v(k) \big ) (x),
\]
with $\cF$ (and $\cF^{-1}$) the Fourier transform (and its inverse). The matrix $\hat{P}_k \in \C^{\dc \times \dc}$ is a tunable Fourier multiplier indexed by $k\in \Z^d$. It is assumed that $\hat{P}_k \equiv 0$ for $|k|_{\ell^\infty}\ge \kappa$, i.e. for wavenumbers $k$ above a specified Fourier cut-off parameter $\kappa$. This Fourier cut-off $\kappa$ is a second hyperparameter of the FNO architecture. We collect the values for different $k\in \Z^d$, $|k|_{\ell^\infty} < \kappa$, in a tensor $\hat{P} = \{ \hat{P}_k \}_{|k|_{\ell^\infty}< \kappa} \in \C^{(2\kappa-1)^d \times \dc \times \dc}$, which acts on the Fourier coefficients $\hat{v}(k) = \cF(v)(k)$, by 
\[
(\hat{P} \hat{v})(k)_i := \sum_{j=1}^\dc \hat{P}_{k,ij} \hat{v}(k), \quad (k \in \Z^d, \; |k|_{\ell^\infty}<\kappa).
\]
The resulting FNO architecture depends on the channel width $\dc$, Fourier cut-off parameter $\kappa$ and depth $L$. We collect all tunable parameters in a vector $\theta \in \R^{q}$. Any parameter $\theta \in \R^{q}$ can be decomposed layer-wise, as 
\[
\theta = (\theta_{L+1}, \theta_L, \dots, \theta_1, \theta_0),
\]
where 
\[
\theta_\ell = 
\set{
W^{(\ell)}_{ij}, \hat{P}_{k,ij}^{(\ell)}, \hat{b}_k^{(\ell)}
}{
i,j = 1,\dots, \dc, \, |k| < \kappa, \, k\in \Z^d
},
\]
collects the parameters of the $\ell$-th hidden layer, for $1\le \ell \le L$. We denote by $\theta_0 = \set{P_{ij}}{i,j = 1,\dots, \dc}$ the parameters of the projection $P$ and by $\theta_{L+1} = \set{Q_{ij}}{i,j = 1,\dots, \dc}$ the parameters of lifting $Q$. Assuming that $d_{\mathrm{in}}, d_{\mathrm{out}} \le d_c$, the dimension of $\theta \in \R^{q}$ satisfies,
\begin{align}
\label{eq:dtheta}
q
=
\dc d_{\mathrm{in}} + L(\dc^2 + (2\kappa)^d \dc^2 + \dc) + \dc d_{\mathrm{out}}
\le
5 (2\kappa)^d L \dc^2
\le 5q.
\end{align}
Consistent with practical implementations, it is generally assumed that the hidden channel dimension of the FNO is at least as large as both the input and output dimensions $d_{\mathrm{in}}, d_{\mathrm{out}}$. We include a list of hyperparameters in Table \ref{tab:1} to aid clarify notation.

\begin{remark}
Since we are interested in a restricted class of operators $\cG: L^2(D) \to \R$, with real-valued outputs, we will replace the general output layer $\cQ: \cV(D;\R^\dc) \to \cY(D;\R^{d_{\mathrm{out}}})$ by a spatially averaged, real-valued version $\tQ: \cV(D;\R^\dc) \to \R$,
\[
\tQ v := \fint_{D} \cQ v(x) \, dx.
\]
This does not affect the parameter-count, while ensuring real-valued outputs. We will refer to this as an output-averaged FNO. 
\end{remark}

In passing and in connection with the last remark, we mention relevant work considering variants of FNO for finite-dimensional input and or output spaces \cite{huang2024operator}, where similar alterations to the original FNO architecture have been studied in greater detail.

\begin{table}
\centering
\begin{tabular}{c|l}
Symbol & Meaning \\
\hline
$\dc$ & channel width \\
$\kappa$ & Fourier cut-off \\
$L$ & depth \\
$q$ & total number of parameters \\
$M$ & parameter bound, $\Vert \theta \Vert_{\ell^\infty} \le M$ 
\end{tabular}
\caption{Summary of (hyper-)parameters of the FNO architecture. This notation is used throughout this subsection.}
\label{tab:1}
\end{table}

\paragraph{Generic curse of parametric complexity for FNO}

Our main theorem will be based on Proposition \ref{prop:uniform-generic}, and establishes a generic curse of parametric complexity for FNO. In contrast to the aforementioned proposition, this theorem holds at the level of continuous real-valued parameters $\theta \in \R^q$, without requiring specification of a bit-encoding. Instead, we assume a mild bound on the parameters $\theta \in \R^q$. We note that similar assumptions have been considered in the recent work \cite{NLM2024data}, to define relevant approximation spaces of FNO. To this end, we make the following definition: 
\begin{definition}
Given an operator $\cG: L^2(D) \to \R$ and $\gamma > 0$, we will say that $\cG$ \textbf{can be approximated by FNO at a logarithmic rate $\gamma > 0$}, if there exists a sequence $\{\Phi_q\}_{q\in \N}$ of output-averaged FNO architectures $\Phi_q: L^2(D) \times \R^{q} \to \R$ with at most $q$ tunable parameters, and a sequence of parameters $\theta_q \in \R^q$, satisfying bound
\[
\Vert \theta_q \Vert_{\ell^\infty} \le \exp(q),
\]
and
\[
\Vert \cG - \Phi_q(\slot;\theta_q) \Vert_{C(\cK)} = O(\log(q)^{-\gamma}), \quad (q\to \infty).
\]
\end{definition}
\begin{remark}
The specific upper bound on the weights, $\Vert \theta_q \Vert_{\ell^\infty} \le \exp(q)$, is here chosen for simplicity. For the following discussion, it could readily be replaced by a more general upper bound, $\Vert \theta_q \Vert_{\ell^\infty} \le c_1\exp(c_2 q^{c_3})$ for fixed constants $c_1$, $c_2$, $c_3$, without affecting the main conclusions.
\end{remark}

We can now state our main result for FNO:
\begin{theorem}
\label{thm:generic-FNO}
Let $\cK\subset L^2(D)$ be compact. Assume that the metric entropy of $\cK$ satisfies an algebraic lower bound, $\cH(\cK;\epsilon)_{L^2(D)} \gtrsim \epsilon^{-1/\alpha}$ for some $\alpha > 0$. Consider FNO with a fixed Lipschitz continuous activation function $\sigma$.
Then generic $\cG \in \Lip_1(\cK)$ cannot be approximated by FNO at a logarithmic rate $\gamma$, for any $\gamma > \alpha$.
\end{theorem}
Thus, loosely speaking and under mild growth assumptions on the weights, the approximation of generic $\cG \in \Lip_1(\cK)$ to accuracy $\epsilon > 0$, requires an FNO architecture with exponentially many tunable parameters in $\epsilon^{-1}$. 

The following corollary is obtained by taking $\cK = \cU(H^s(\T^d))$ as the unit ball in a Sobolev space $H^s(\T^d)$ for $s>0$, and with $\T^d$ the $d$-dimensional periodic torus:
\begin{corollary}
Let $s>0$, and denote $\cK = \cU(H^s(\T^d))$. Then generic $\cG \in \Lip_1(\cK)$ cannot be approximated by FNO at logarithmic rate $\gamma$, for any $\gamma > s/d$. 
\end{corollary}

\begin{proof}[Proof of Theorem \ref{thm:generic-FNO}]
Fix $\gamma > \alpha$. We wish to show that generic $\cG \in \Lip_1(\cK)$ cannot be approximated at logarithmic rate $\gamma$. Proof of this claim will make use of the following lemma:

\begin{restatable}[FNO quantization lemma]{lemma}{fnoquant}
\label{lem:quantize}
Fix Lipschitz continuous activation function $\sigma$.
Let $\gamma > 0$. For any $q\in \N$, there exists a quantized neural operator $\tilde{\Phi}_{n_q}: L^2(D) \times \{0,1\}^{n_q} \to \R$ with $2^{n_q}$ quantized parameter values, where $n_q \asymp q^{m}$, $m = d+6$, such that for any output-averaged FNO $\Phi_q$ with activation $\sigma$ and at most $q$ tunable parameters, we have
\[
\sup_{\theta \in {[-M_q,M_q]}^q} \inf_{[\theta]\in \{0,1\}^{n_q}} \Vert \Phi_q(\slot; \theta) - \tilde{\Phi}_{n_q}(\slot;[\theta]) \Vert_{C(\cK)} 
\le 
\log(q)^{-\gamma}.
\]
where $M_q := \exp(q)$.
\end{restatable}

\begin{proof}[Brief sketch of proof]
The detailed proof of this lemma is included in Appendix \ref{app:quantize}; in short, the proof relies on two observations: (i) all possible FNO architectures with at most $q$ parameters can be encapsulated by a ``super'' FNO-architecture $\hat{\Phi}(\slot;\theta)$ with a number of parameters that is bounded algebraically in $q$ for fixed algebraic exponent, and (ii) quantization of this super-architecture with an algebraically bounded number of bits is possible, since the mapping $\theta \mapsto \hat{\Phi}(\slot;\theta)$ has at least a weak form of stability (Lipschitz continuity) over the relevant range of parameters $\theta$, and a Lipschitz constant that grows at a sufficiently slow rate as a function of $q$.
\end{proof}

By Lemma \ref{lem:quantize}, there exists $m\in \N$, a sequence $n_q \asymp q^{m}$, and a sequence of quantized neural operators, $\tilde{\Phi}_{n_q}: L^2(D)\times \{0,1\}^{n_q} \to \R$, such that
\[
\sup_{\theta \in [-M_q,M_q]^q} \inf_{[\theta]\in \{0,1\}^{n_q}} \Vert \Phi_q(\slot; \theta) - \tilde{\Phi}_{n_q}(\slot;[\theta]) \Vert_{C(\cK)} 
\le 
\log(q)^{-\gamma}.
\]

Associated with this subsequence $n_q \to \infty$, we now define an (abstact) sequence of bit-encoded neural operators for arbitrary $n\in \N$; specifically, we define $\tilde{\Phi}_{n}(\slot;\slot): L^2(D) \times \{0,1\}^n \to \R$, by
\[
\tilde{\Phi}_{n}(\slot;[\theta]_n) := \tilde{\Phi}_{n_q}(\slot;[\theta]_{n_q}),
\quad [\theta]_n \in \{0,1\}^n,
\]
where $n_q$ is chosen maximal such that $n_q \le n$, and $[\theta]_{n_q}$ are the first $n_q\le n$ bits of $[\theta]_n$ (the values of the remaining bits are simply ignored). We note that since $n_q \asymp q^{m}$, we have $\log(q)^{-\gamma} \asymp \log(n_q)^{-\gamma}$. Furthermore, for arbitrary fixed operator $\cG$, we note that the decay
\begin{align*}
\inf_{[\theta]\in \{0,1\}^{n_q}} \Vert \cG - \tilde{\Phi}_{n_q}(\slot;[\theta]) \Vert_{C(\cK)} 
\lesssim
\log(n_q)^{-\gamma},
\end{align*}
along the specified subsequence $n_q \asymp q^m$ also implies the error decay 
\begin{align}
\label{eq:to-full}
\inf_{[\theta]\in \{0,1\}^{n}} \Vert \cG - \tilde{\Phi}_{n}(\slot;[\theta]) \Vert_{C(\cK)} 
\lesssim
\log(n)^{-\gamma},
\end{align}
along the full sequence $n\in \N$, as $n\to \infty$. This is immediate from the definition of $\tilde{\Phi}_n$ and the fact that $n_q \le n < n_{q+1}$ does not leave exponential gaps between subsequent $n_q$, since $1 \le n_{q+1}/n_q \asymp (q+1)^m / q^m = O(1)$; in particular, this implies that $\log(n_q) \sim \log(n_{q+1}) \sim \log(n)$.

By Proposition \ref{prop:uniform-generic}, the set of operators $\MM \subset \Lip_1(\cK)$ which can be approximated by such a sequence $\{\tilde{\Phi}_n\}$, at logarithmic rate $\gamma$, is meagre (its complement is residual).  To conclude the argument, it therefore suffices to show that if $\cG$ can be approximated by FNO at logarithmic rate $\gamma$, then $\cG \in \MM$. This then implies that the set of operators that can be approximated by FNO at logarithmic rate $\gamma$ is a subset of $\MM$, and hence is itself meagre.

To this end, assume that $\cG \in \Lip_1(\cK)$ is approximated by FNO at logarithmic rate $\gamma$. By definition, there exists a sequence of FNOs, $\Phi_q: L^2(D)\times \R^q\to \R$, such that,
\[
\inf_{\theta \in [-M_q,M_q]^q}
\Vert \cG - \Phi_q(\slot;\theta) \Vert_{C(\cK)}
= O(\log(q)^{-\gamma}).
\]
By the triangle inequality,
\begin{align*}
&\inf_{[\theta]\in\{0,1\}^{n_q}}
\Vert \cG - \tilde{\Phi}_{n_q}(\slot; [\theta]) \Vert_{C(\cK)} 
\\
&\hspace{2cm}
\le
\Vert \cG - \Phi_{q}(\slot; \theta_q) \Vert_{C(\cK)}
+
\inf_{[\theta]\in\{0,1\}^{n_q}}
\Vert \Phi_q(\slot;\theta_q) - \tilde{\Phi}_{n_q}(\slot; [\theta]) \Vert_{C(\cK)}
\\
&\hspace{2cm}
\le O( \log(q)^{-\gamma}) + O(\log(n_q)^{-\gamma}) = O(\log(n_q)^{-\gamma}),
\end{align*}
along the specified sequence $n_q \to \infty$. By \eqref{eq:to-full}, this implies that
\[
\Vert \cG - \tilde{\Phi}_{n}(\slot; [\theta]) \Vert_{C(\cK)} = O(\log(n)^{-\gamma}),
\]
along the entire sequence $n\to \infty$,
and hence $\cG \in \MM$, i.e. $\cG$ belongs to the meagre set of operators which can be approximated by the sequence $\{\tilde{\Phi}_n\}$ at logarithmic rate $\gamma$. 

We have shown that any operator $\cG$ that is approximated by FNO at logarithmic rate $\gamma$ belongs to the meagre set $\MM$. Hence, the set of operators that is approximated by FNO at logarithmic rate $\gamma$ is itself meagre, and its complement $\RR = \Lip_1(\cK) \setminus \MM$ is residual. We conclude that generic operators $\cG \in \Lip_1(\cK)$, belonging to $\RR$, cannot be approximated at logarithmic rate $\gamma > \alpha$.
\end{proof}

\section{The metric entropy of Lipschitz operators}
\label{sec:3}

In the present section, we provide lower bounds on the metric entropy of Lipschitz operators in two general settings; the first pertains to the sup-norm over a compact set of inputs, the second is of relevance to the approximation with respect to the Bochner $L^p$-norm with respect to a probability measure on the input space. After briefly recalling the relation between covering and packing numbers, we proceed to consider the sup-norm setting in Section \ref{sec:sup-norm} and the $L^p$-setting in Section \ref{sec:Lp-norm}.

\subsection{Entropy, covering and packing}
We recall from Definition \ref{def:entropy} that the metric entropy $\cH(\A;\epsilon)_\V$ of a subset $\A\subset \V$ is defined by $\cH(\A;\epsilon)_\V = \log_2 \cN(\A;\epsilon)_\V$; here, $\cN(\A;\epsilon)_\V$ denotes the covering number of $\A$, which is here defined as the smallest number of closed balls of radius $\epsilon$ needed to cover $\A$. We also recall the closely related notion of a packing number:
\begin{definition}[Packing number]
Let $(\V,d)$ be a metric space. The packing number of a subset $\A\subset \V$, denoted $\cM(\A;\epsilon)_{\V}$, is the largest integer $M \in \N$ for which there exist elements $u_1,\dots, u_M  \in \A$, with pairwise distance $d(u_j,u_k) \ge \epsilon$, for all distinct $j,k\in \{1,\dots, M\}$.
\end{definition}

With our definitions, the following inequalities between covering and packing numbers are elementary: For any subset $\A\subset \V$, we have
\begin{align}
\label{eq:covering-packing}
\cM(\A; 3\epsilon)_{\V} \le \cN(\A;\epsilon)_{\V} \le \cM(\A;\epsilon)_\V. 
\end{align}
We mention that, if the covering number was defined by open balls, the factor $3$ in the first term could have been replaced by $2$. With our definition in terms of closed balls, any factor $>2$ would do -- we here choose $3$ for simplicity. 

\subsection{Uniform approximation}
\label{sec:sup-norm}

We are here interested in the uniform setting (Setting \ref{set:uniform}), i.e. the unifrom approximation of a (real-valued) mapping $\cG: \cK \to \R$ over a compact domain $\cK \subset \cX$. 

% We also make the following remark on the choice of output space.

% \begin{remark}
% The output space plays a subordinate role in our analysis. In fact, since we are interested in lower bounds, the output space $\R$ could have been replaced by an arbitrary Banach space $\cY$. This relies on the fact that the metric entropy of $\Lip_1(\cK; \cY)$, consisting of all Lipschitz is always lower bounded by $\Lip_1(\cK)$, provided that $\dim(\cY)\ge 1$. As a word of caution, we note that the metric entropy of $\Lip_1(\cK;\cY)$ is infinite, whenever $\dim(\cY) = \infty$; indeed, in this case, $\Lip_1(\cK;\cY)$ is not compact and the covering numbers are not finite. To avoid this, one could instead consider $\Lip_1(\cK;\cK')$, where $\cK'\subset \cY$ is compact. However, this would require in-depth discussion of good additional assumption on $\cK'$. Therefore, to simplify our exposition, we will consider only the case $\cY = \R$ (or, equivalently, $\cK' = [-1,1]$), in this work. 
% \end{remark}

As pointed out before, given the link between minimax code-length and metric entropy, we are interested in estimating the metric entropy of $\Lip_1(\cK)$ for $\cK$ a compact metric space. The following proposition relates the metric entropy of $\Lip_1(\cK)\subset \V$ to that of $\cK$, when $\V = C(\cK)$ is metrized by the sup-norm:
\begin{proposition}
\label{prop:sup-norm}
Let $(\cK,d)$ be a metric space. Let $\epsilon \in (0,1/3]$. 
The metric entropy of $\Lip_1(\cK) \subset C(\cK)$ is lower bounded by
\begin{align}
\cH(\Lip_1(\cK), \epsilon)_{C(\cK)} \ge 2^{\cH(\cK;6\epsilon)_\cX}.
\end{align}
\end{proposition}

Proposition \ref{prop:sup-norm} shows that the space of $1$-Lipschitz functions on a compact metric space has \emph{exponentially larger} entropy than the underlying space. 
\begin{proof}
Let $\epsilon \in (0,1/3]$ be given. Let $N = \cN(\cK; 6\epsilon)_\cX$. Since the covering number lower bounds the packing number (cf. \eqref{eq:covering-packing}), there exist $N$ elements $u_1, \dots, u_N \in \cK$, with pairwise distance $\ge 6\epsilon$. Let 
\[
\psi_j(u) := \max(3\epsilon - d(u,u_j), 0), \quad j=1,\dots, N,
\]
denote ``hat'' functions centered at $u_j$, and non-vanishing only on $B_{3\epsilon}(u_j)\subset \cK$.
We note that each $\psi_j$ is $1$-Lipschitz, satisfies $\Vert \psi_j \Vert_{C(\cK)} = 3\epsilon$, and the supports of $\psi_j$ are essentially disjoint.

We now consider the set of Lipschitz functions $f: \cK \to \R$ of the form,
\[
f_\sigma(u) = \sum_{j=1}^N \sigma_j \psi_j(u), \quad \sigma = (\sigma_1,\dots, \sigma_N) \in \{0,1\}^N. 
\]
These functions satisfy $\Vert f_\sigma \Vert_{C(\cK)} \le 3\epsilon \le 1$, and $\Lip(f_\sigma) \le \max_{j=1,\dots, N} \Lip(\psi_j) = 1$, for all choices of $\sigma$. Furthermore, if $\sigma, \sigma' \in \{0,1\}^N$ are two distinct elements, say with $\sigma_{j_0} \ne \sigma'_{j_0}$, then it is straightforward to show that $\Vert f_\sigma - f_{\sigma'} \Vert_{C(\cK)} \ge \Vert \psi_{j_0} \Vert_{C(\cK)} = 3\epsilon$. 

Thus, we have shown that there exist $2^N = |\{0,1\}^N|$ functions $f_\sigma \in \Lip_1(\cK)$, with pairwise $C(\cK)$-distance $\ge 3\epsilon$. In particular, this implies that the packing number $\cM(\Lip_1(\cK);3\epsilon)_{C(\cK)} \ge 2^N$, and by the inequality \eqref{eq:covering-packing} between packing- and covering-numbers, this now implies that 
\[
\cN(\Lip_1(\cK);\epsilon)_{C(\cK)} \ge \cM(\Lip_1(\cK);3\epsilon)_{C(\cK)} \ge 2^N.
\]
The claim follows by taking logarithms and recalling that $N = \cN(\cK;6\epsilon) = 2^{\cH(\cK;6\epsilon)_{\cX}}$.
\end{proof}

We conclude this section with several corollaries of Proposition \ref{prop:sup-norm}.

\begin{corollary}[Lipschitz functions on finite-dimensional domains]
If $D \subset \R^d$ is a compact domain in Euclidean space, then
\[
\cH(\Lip_1(D);\epsilon) \gtrsim \epsilon^{-d}.
\]
\end{corollary}

\begin{proof}
It is a well-known fact that
\[
\cN(D;\epsilon) \gtrsim \epsilon^{-d},
\]
with an implied constant depending on the dimension $d$ and the volume of $D$; for example, this can be a simple volume argument for an $\epsilon$-covering $D \subset \bigcup_{n=1}^N \overline{B_\epsilon}(x_n)$, which yields
\[
\vol(D) 
\le \vol\left( \bigcup_{n=1}^N \overline{B_\epsilon}(x_n) \right)
\le N \vol(\overline{B_\epsilon}) = N C_d \epsilon^{d}
\; \Rightarrow \; 
N \ge \frac{\vol(D)}{C_d \epsilon^{d}}. 
\]
The claim thus follows from Proposition \ref{prop:sup-norm}.
\end{proof}

\begin{corollary}[Lipschitz functionals on Sobolev spaces]
Let $D \subset \R^d$ be a compact domain in Euclidean space. Let $\cK = \cU(W^{s,p}(D))$ be the unit ball in the space of Sobolev functions possessing $s>0$ weak derivatives in $L^p(D)$, considered as a subset of $L^p(D)$. Then  there exists a constant $c>0$, such that
\[
\cH(\Lip_1(\cK);\epsilon) \gtrsim \exp(c\epsilon^{-d/s}). 
\]
\end{corollary}

\begin{proof}
The metric entropy of $\cU(W^{s,p}(D))$ with respect to the $L^p$-norm is lower bounded by \cite{birman1966approximation}:
\[
\cH(\cK;\epsilon)_{L^p(D)} \gtrsim \epsilon^{-d/s}.
\]
The claim thus follows from Proposition \ref{prop:sup-norm}.
\end{proof}

\begin{corollary}[Lipschitz functionals on H\"older spaces]
Let $D \subset \R^d$ be a compact domain in Euclidean space. Let $\cK = \cU(C^{s}(D))$ be the unit ball in the space of H\"older continuous functions of order $s>0$, considered as a subset of $C(D)$. Then there exists a constant $c>0$, such that
\[
\cH(\Lip_1(\cK);\epsilon) \gtrsim \exp(c\epsilon^{-d/s}). 
\]
\end{corollary}

\begin{proof}
The metric entropy of $\cU(C^s(D))$ with respect to the sup-norm is lower bounded by \cite{kolmogorov}:
\[
\cH(\cK;\epsilon)_{C(D)} \gtrsim \epsilon^{-d/s}.
\]
The claim thus follows from Proposition \ref{prop:sup-norm}.
\end{proof}

\subsection{Approximation in expectation}
\label{sec:Lp-norm}

Besides the setting discussed in the previous section, which is relevant for the uniform approximation of operators over a compact set of input functions, another commonly studied setting is the approximation in expectation (cp. Setting \ref{set:expectation}): Here, we consider $1$-Lipschitz mappings $\cG: \cX \to \R$ defined on a separable Hilbert space $\cX$. We fix a probability measure $\mu$ on $\cX$ and consider inputs as random draws $u\sim \mu$. We assume that $\mu$ satisfies the minimal structural Assumption \ref{ass:prob}; under this assumption, random draws $u\sim \mu$ can be obtained from a Karhunen-Loeve-like expansion,
$u = \sum_{j=1}^\infty \sqrt{\lambda_j} Z_j e_j$.

Our aim is to find lower bounds on the metric entropy of $\Lip_1(\cX)\subset \V$, where $\V = L^p(\mu)$ is the space of $L^p(\mu)$-integrable operators.
The following entropy estimate represents the main novel contribution of this section:
\begin{proposition}
\label{prop:prob}
Let $\cX$ be a separable Hilbert space, and let $\mu$ be a probability measure satisfying Assumption \ref{ass:prob}. Let $p\in [1,\infty)$ be given. Assume that the coefficients $\sqrt{\lambda_j} \gtrsim j^{-\alpha}$ as $j\to \infty$, where $\alpha > 0$. Then the metric entropy of $\Lip_1(\cX)$ with respect to the Bochner $L^p(\mu)$-norm, obeys the following lower bound: There exist constants $c,\epsilon_0>0$, such that
\begin{align}
\cH(\Lip_1(\cX); \epsilon)_{L^p(\mu)} \ge \exp\left(c\epsilon^{-1/(\alpha+1)} \right),
\quad \forall \, \epsilon \in (0,\epsilon_0].
\end{align}
\end{proposition}

Our proof of Proposition \ref{prop:prob} will rely on several technical lemmas, which we state and prove below. The first lemma identifies an isometric embedding $L^p([0,1]^d) \embeds L^p(\mu)$.
\begin{lemma}
\label{lem:embed-Lp}
Let $\cX$ be a separable Hilbert space.
Let $\mu \in \cP(\cX)$ satisfy Assumption \ref{ass:prob}, and let $p\in [1,\infty)$. Then for any $d\in \N$, there exists an isometric embedding,
\begin{align}
\iota_d: L^p([0,1]^d) \embeds L^p(\mu),
\end{align}
such that $\iota_d(\Lip_1([0,1]^d)) \subset \Lip_{L/\sqrt{\lambda_d}}(\cX)$, where the Lipschitz norm on $[0,1]^d$ is defined with respect to the $\ell^\infty$-norm on $[0,1]^d$.
\end{lemma}

\begin{proof}
By assumption, $\mu\in \cP(\cX)$ is the law of a random field $u: \Omega \to \cX$ of the form,
\begin{align}
\label{eq:rv}
u(\omega) = \sum_{j=1}^\infty \sqrt{\lambda_j} Z_j(\omega) e_j,
\end{align}
with $Z_j$ independent, $Z_j \sim \rho_j(z) \, dz$. To construct the claimed isometry, we define $F_j(z) := \int_{-\infty} \rho_j(\zeta) \, d\zeta$ as the cumulative distribution function of $\rho_j$. We recall that $F_j(Z_j) \sim \cU(0,1)$ is uniform $[0,1]$ distributed. Furthermore, we clearly have $\Lip(F_j) = \Vert \rho_j \Vert_{L^\infty(\R)} \le L$, where the last bound is by Assumption \ref{ass:prob}.

Given $u\in \cX$, we define $u_j := \langle e_j, u\rangle_{\cX}$ the coefficients of $u$ with respect to the orthonormal basis $\{e_j\}$. Using the CDFs introduced above, $F_j: \R \to [0,1]$, we now define a mapping,
\[
\iota_d: L^p([0,1]^d) \to L^p(\mu), 
\quad 
(\iota_d f)(u) := f(F_1(u_1/\sqrt{\lambda_1}), \dots, F_d(u_d/\sqrt{\lambda_d})).
\]
To see that this is well-defined, we note that, using the expansion of the random field \eqref{eq:rv}, $u_j/\sqrt{\lambda_j} = Z_j$, and hence
\[
(\iota_d f)(u) = f(F_1(Z_1), \dots, F_d(Z_d)), \quad \text{for } u \sim \mu,
\]
and we once again remind ourselves that $F_j(Z_j) \sim \cU(0,1)$ is uniformly distributed on $[0,1]$, and that the $Z_j$ are independent by assumption. Thus, it follows that
\begin{align*}
\E_{u\sim \mu} | (\iota_d f)(u) |^p
&= \E |f(F_1(Z_1), \dots, F_d(Z_d))|^p
\\
&= \int_{[0,1]^d} |f(x_1,\dots, x_d)|^p \, dx 
\\
&= \Vert f \Vert_{L^p([0,1]^d)}^p.
\end{align*}
Thus, $\Vert \iota_d f \Vert_{L^p(\mu)} = \Vert f \Vert_{L^p([0,1]^d)}$. This shows that $\iota_d: L^p([0,1]^d) \to L^p(\mu)$ is an isometry as claimed. To verify that $\iota_d(\Lip_1([0,1]^d)) \subset \Lip_{L/\sqrt{\lambda_d}}(\cX)$, we note that 
\[
h_d: (\cX, \Vert \slot \Vert_{\cX}) \to ([0,1]^d, \ell^\infty), \quad u \mapsto (F_1(u_1/\sqrt{\lambda_1}), \dots, F_d(u_d/\sqrt{\lambda_d})),
\]
has Lipschitz constant bounded by 
\[
\Lip(h_d) \le \max_{j=1,\dots, d} \frac{\Lip(F_j)}{\sqrt{\lambda_j}} \le \frac{L}{\sqrt{\lambda_d}}.
\]
Thus, for any $f\in \Lip_1([0,1]^d) =  \Lip_1(([0,1]^d,\ell^\infty))$, 
\[
\Lip(\iota_d f) = \Lip( f \circ h_d ) \le \Lip(f) \Lip(h_d) \le \frac{L}{\sqrt{\lambda_d}}.
\]
Furthermore, we also have $\Vert \iota_d f \Vert_{C(\cX)} \le \Vert f \Vert_{C([0,1]^d)} \le 1$. This shows that 
\[
\Vert \iota_d f \Vert_{\Lip} 
= \max\left\{
\Vert \iota_d f \Vert_{C(\cX)}, \Lip(\iota_d f)
\right\}
\le \max\left\{1,\frac{L}{\sqrt{\lambda_d}}\right\} = \frac{L}{\sqrt{\lambda_d}}.
\]
Here, we have made use of the choice $L>\sqrt{\lambda_1} \ge \sqrt{\lambda_d}$ (cp. \eqref{eq:rhobd}) in the last inequality. This concludes our proof.
\end{proof}

As a consequence of Lemma \ref{lem:embed-Lp}, we have:
\begin{corollary}
\label{cor:embed-Lp}
Under the assumptions of Lemma \ref{lem:embed-Lp}, we have
\begin{align}
\label{eq:embed-Lp}
\cH(\Lip_1(\cX); \epsilon)_{L^p(\mu)} 
\ge
\cH\left(\Lip_1([0,1]^d); \frac{L\epsilon}{\sqrt{\lambda_d}} \right)_{L^p([0,1]^d)},
\end{align}
for any $d \in \N$.
\end{corollary}

\begin{proof}
We recall the existence of an isometric embedding $\iota_d: L^p([0,1]^d) \to L^p(\mu)$ from Lemma \ref{lem:embed-Lp}, with $\iota_d(\Lip_1(\cX)) \subset \Lip_{L/\sqrt{\lambda_d}}([0,1]^d)$. It follows that 
\begin{align*}
\cN(\Lip_1(\cX); \epsilon)_{L^p(\mu)} 
&= 
\cN(\Lip_{L/\sqrt{\lambda_d}}(\cX); L\epsilon/\sqrt{\lambda_d})_{L^p(\mu)} 
\\
&\ge 
\cN(\iota_d(\Lip_1([0,1]^d)); L\epsilon/\sqrt{\lambda_d})_{L^p(\mu)}
\\
&= 
\cN(\Lip_{1}([0,1]^d); L\epsilon/\sqrt{\lambda_d})_{L^p([0,1]^d)}.
\end{align*}
Taking logarithms, the claimed inequality between the metric entropy follows.
\end{proof}

The proof of Proposition \ref{prop:prob} will furthermore make use of the following result in the finite-dimensional setting:
\begin{lemma}
\label{lem:prob}
Let $p\in [1,\infty)$ be given. For $d\in \N$, consider $\Lip_1([0,1]^d) \subset L^p([0,1]^d)$. Then there exists a constant $c>0$, independent of $d$, such that we have the following lower bound on the metric entropy:
\begin{align}
\label{eq:prob}
\cH(\Lip_1([0,1]^d); \epsilon)_{L^p([0,1]^d)} \ge \frac18 \left(\frac{c}{d \epsilon}\right)^{d},
\quad \forall \, \epsilon \in \left(0,\frac{c}{d}\right].
\end{align}
\end{lemma}

\begin{proof}
Since the H\"older inequality implies, for any $p\in [1,\infty)$, that $\Vert f \Vert_{L^1([0,1]^d)} \le \Vert f \Vert_{L^p([0,1]^d)}$, it follows that any covering of $\Lip_1([0,1]^d)$ by $\epsilon$-balls with respect to the $L^p$-norm, also gives rise to a covering of $\Lip_1([0,1]^d)$ with respect to the $L^1$-norm (with the same centers). In particular, this implies that 
\[
\cN(\Lip_1([0,1]^d);\epsilon)_{L^p([0,1]^d)} \ge \cN(\Lip_1([0,1]^d);\epsilon)_{L^1([0,1]^d)},
\]
and we only need to establish \eqref{eq:prob} for $p=1$.

For $\lambda \in (0,1)$, define $\phi_\lambda: [0,1]^d \to \R_+$ as a composition $g_\lambda \circ \Vert \slot \Vert_{\ell^\infty}$, where $g_\lambda: \R \to \R$ is a piecewise linear function (approximately $g_\lambda \approx 1_{[0,1]}$) with values,
\[
g_\lambda(x) :=
\begin{cases}
0, &(x\notin [0,1]), \\
1, &(x\in [\lambda/2,1-\lambda/2]),
\end{cases}
\]
and $g_\lambda$ interpolates linearly between $0$ and $1$ on $[0,\lambda/2]$, and from $1$ to $0$ on $[1-\lambda/2,1]$. By construction, $g_\lambda$ is $2/\lambda$-Lipschitz. Since $x \mapsto \Vert x \Vert_{\ell^\infty}$ is $1$-Lipschitz, it follows that $\Lip(\phi_\lambda) = \Lip(g_\lambda \circ \Vert \slot \Vert_{\ell^\infty}) \le 2/\lambda$. Clearly, smaller $\lambda$ leads to a larger Lipschitz constant. However, by construction of $\phi_\lambda$, we have $\phi_\lambda \ge 1_{[\lambda/2,1-\lambda/2]^d}$. In particular, this implies that $\Vert \phi_{\lambda} \Vert_{L^1} \ge (1-\lambda)^{d}$. Thus, smaller $\lambda$ increases the $L^1$-norm of $\phi_\lambda$.

Given $N\in \N$, we now subdivide $[0,1]^d$ into $N^d$ cubes of equal length, indexed by $j\in [N]^d$, where $[N]^d = \{1,\dots, N\}^d$. For any multi-index $j\in [N]^d$, we define $\phi_{\lambda, j}(x)$ as a rescaled and translated copy of $\phi_{\lambda}$, such that the support of $\phi_{\lambda, j}$ coincides with the $j$-th cube. In particular, by construction of $\phi_\lambda$, this implies that 
\begin{align}
\Vert \phi_{\lambda,j} \Vert_{L^1([0,1]^d)}^2 &\ge (1-\lambda)^{d} N^{-d}, \\
\Lip(\phi_{\lambda,j}) &\le 2N\lambda^{-1}.
\end{align}
We also note that the $\phi_{\lambda, j}$ have essentially disjoint supports.
For $\sigma \in \{-1,1\}^{[N]^d}$, we now define 
\[
f_\sigma(x)
=
\frac{\lambda}{2N}\sum_{j \in [N]^d} \sigma_j \phi_{\lambda,j}(x).
\]
The factor in front of the sum ensures that $\Lip(f_\sigma) \le 1$. Furthermore, we also note that $\Vert f_\sigma \Vert_{C([0,1]^d)} \le \lambda / 2N \le 1$ for any choice of $\lambda \in (0,1)$ and $N \in \N$. In particular, we have $f_\sigma \in \Lip_1([0,1]^d)$, for any choice of $\sigma$. We finally observe that, due to the disjoint supports of the $\phi_{\lambda,j}$, we have, for any $\sigma,\sigma' \in \{-1,1\}^{[N]^d}$,
\begin{align*}
\Vert f_\sigma - f_{\sigma'} \Vert_{L^1([0,1]^d)}
&= 
\frac{\lambda}{2N} \sum_{j\in [N]^d} |\sigma_j - \sigma'_j| \Vert \phi_{\lambda,j}\Vert_{L^1([0,1]^d)}
\\
&\ge
\lambda (1-\lambda)^d N^{-1} \frac{\#\{\sigma_j\ne \sigma_j'\}}{N^{d}}. 
\end{align*}
The last quotient is the fraction of entries in which $\sigma$ and $\sigma'$ differ. It turns out that there exists a subset $\Xi \subset \{-1,1\}^{[N]^d}$, such that any $\sigma \ne \sigma'$ belonging to $\Xi$ differ on a substantial fraction of their components; more precisely, as noted in \cite{achour2022a} as a result of the Gilbert-Varshamov bound, there exists a subset $\Xi \subset \{-1,1\}^{[N]^d}$ satisfying that any two distinct elements $\sigma,\sigma' \in \Xi$, differ on at least a fourth of their coordinates,
\begin{align}
\frac{\#\{\sigma_j\ne \sigma_j'\}}{N^{d}} \ge \frac{1}{4},
\qquad \forall \, \sigma, \sigma' \in \Xi, \; \sigma \ne \sigma',
\end{align}
and the cardinality of $\Xi$ is lower bounded by,
\begin{align}
\label{eq:size-Xi}
\#\Xi \ge \exp(N^d/8) \ge 2^{N^d/8}.
\end{align}
This implies that for any two $\sigma\ne \sigma'$ in $\Xi$, we have
\[
\Vert f_\sigma - f_{\sigma'} \Vert_{L^1([0,1]^d)}
\ge \frac1{4N}\lambda (1-\lambda)^{d}.
\]
Optimizing the right-hand side over $\lambda \in (0,1)$, we set $\lambda = 1/(1+d)$ to obtain,
\[
\Vert f_\sigma - f_{\sigma'} \Vert_{L^1([0,1]^d)}
\ge \frac{1}{4(d+1)N} \frac{1}{(1+1/d)^{d}}
\ge \frac{1}{4e(d+1)N}
\ge \frac{1}{8edN},
\]
where we used that the Euler constant $e \ge (1+1/d)^d$ and the fact that $d\ge 1$ implies $d+1 \le 2d$ in the last bound. 

Taking into account the bound \eqref{eq:size-Xi}, it follows that the packing number $\cM(\Lip_1([0,1]^d); \epsilon)$, satisfies the lower bound,
\[
\log_2\cM(\Lip_1([0,1]^d); (\beta_d N)^{-1}) \ge N^d/8, \quad \forall \, N\in \N,
\]
where we have defined $\beta_d = 8ed$.
Given $\epsilon \in (0,\beta_d^{-1}]$, we can find $N\in \N$, such that 
\[
(\beta_d N)^{-1} \ge \epsilon \ge (2\beta_d N)^{-1}.
\]
It follows that
\begin{align*}
\log_2 \cM(\Lip_1([0,1]^d); \epsilon) 
&\ge \log_2 \cM(\Lip_1([0,1]^d); (2\beta_d N)^{-1})
\\
&\ge \frac{(2N)^d}{8}
\ge \frac{1}8 \left(\frac{\beta_d \epsilon}{2} \right)^{-d}.
\end{align*}
We conclude that
\[
\log_2 \cM(\Lip_1([0,1]^d); \epsilon) 
 \ge \frac18 \left( \frac{\beta_d}{2} \epsilon \right)^{-d},
\quad
\forall \, \epsilon \in (0,\beta_d^{-1}].
\]
This lower bound on the packing number holds for any dimension $d\in \N$. We can now use the general relation $\cN(A;\epsilon) \ge \cM(A;2\epsilon)$ between the covering- and packing-numbers \eqref{eq:covering-packing}, to conclude that,
\[
\cH(\Lip_1([0,1]^d);\epsilon) 
= \log_2 \cN(\Lip_1([0,1]^d); \epsilon) 
 \ge \frac18 \left( \beta_d \epsilon \right)^{-d},
\quad
\forall \, \epsilon \in (0,\beta_d],
\]
where $\beta_d = 8ed$. This proves the claim with $c = 1/(8e)$, i.e.
\[
\log_2 \cN(\Lip_1([0,1]^d); \epsilon) 
 \ge \frac18 \left( \frac{c}{d\epsilon} \right)^{d},
\quad
\forall \, \epsilon \in \left(0,\frac{c}{d}\right],
\]
\end{proof}

Assuming the results of Corollary \ref{cor:embed-Lp} and Lemma \ref{lem:prob}, we can now prove Proposition \ref{prop:prob}.

\begin{proof}[Proof of Proposition \ref{prop:prob}]
Combining the lower bound \eqref{eq:embed-Lp} and \eqref{eq:prob}, we obtain that for any $d\in \N$, 
\begin{align*}
\log_2 \cN(\Lip_1(\cX); \epsilon)_{L^p(\mu)}
\ge \frac18 \left(\frac{c\sqrt{\lambda_d}}{L d\epsilon}\right)^{d}
\ge \left(\frac{c\sqrt{\lambda_d}}{8L d\epsilon}\right)^{d},
\end{align*}
provided that $\epsilon \le \frac{c\sqrt{\lambda_d}}{L d}$.
Since $\lambda_d \gtrsim d^{-2\alpha}$ by assumption, and since $C$ and $L$ are constants independent of $d$, it thus follows that there exist $c_1,c_2 > 0$, independent of $d$, such that
\begin{align}
\label{eq:logN-d}
\log_2 \cN(\Lip_1(\cX); \epsilon)_{L^p(\mu)}
\ge \left(\frac{c_1}{d^{1+\alpha} \epsilon}\right)^{d},
\quad 
\text{if $\epsilon \le c_2 d^{-(1+\alpha)}$.}
\end{align}
The idea is now to choose $d = d(\epsilon) \sim \epsilon^{-1/(\alpha+1)}$, such that the term inside the parentheses is lower bounded by $e^{\beta}$ for some fixed $\beta>0$, implying that the right hand side is $\gtrsim (e^\beta)^d = \exp(\beta d) \gtrsim \exp(c\epsilon^{-1/(\alpha+1)})$ for some constant $c>0$. This then leads to the claimed lower bound. We now proceed to provide the details of the required argument.

We first fix $\beta = -\log(c_2/c_1)$, such that 
\begin{align}
\label{eq:beta-def}
e^{-\beta} = c_2 / c_1.
\end{align}
We next define 
\begin{align}
\label{eq:eps-def}
\epsilon_0 = c_1 e^{-\beta} = c_2.
\end{align}
Since $c_1,c_2$ are independent of $d$, it follows that also $\beta$ and $\epsilon_0$ are independent of $d$.

For any $\epsilon \in (0,\epsilon_0]$, the above choice ensures that
\[
\epsilon \le \epsilon_0 \le c_1 e^{-\beta},
\]
and hence there exists a unique $d = d(\epsilon)\in \N$, such that
\[
\epsilon d^{(1+\alpha)} \le c_1 e^{-\beta} < \epsilon (2d)^{(1+\alpha)}.
\]
In particular, upon rearranging the first inequality in the last display, we obtain the two equivalent formulations,
\begin{gather}
\epsilon \le c_1 e^{-\beta} d^{-(1+\alpha)} = c_2 d^{-(1+\alpha)},
\label{eq:dd1}
\\
\frac{c_1}{d^{(1+\alpha)} \epsilon}
\ge e^{\beta}.
\label{eq:dd2}
\end{gather}
while the second bound $c_1 e^{-\beta} < \epsilon (2d)^{(1+\alpha)}$ implies,
\begin{align}
\label{eq:dd3}
\beta d \ge c \epsilon^{-1/(\alpha+1)}, \quad \text{where } c := \beta \left[\frac{c_1}{2e^\beta}\right]^{1/(\alpha+1)}.
\end{align}
With this choice of $d=d(\epsilon)$, equation \eqref{eq:dd1} guarantees that 
the estimate in \eqref{eq:logN-d} applies to all $\epsilon \in (0,\epsilon_0]$. This in turn implies that 
\begin{align*}
\cH(\Lip_1(\cX);\epsilon)_{L^p(\mu)} &= \log_2 \cN(\Lip_1(\cX);\epsilon)_{L^p(\mu)}
\\
&\explain{\ge}{\eqref{eq:logN-d}} 
\left(\frac{c_1}{d^{(1+\alpha)} \epsilon} \right)^{d}
\\
&\explain{\ge}{\eqref{eq:dd2}} e^{\beta d} 
\\
&\explain{\ge}{\eqref{eq:dd3}} 
\exp\left(
c \epsilon^{-1/(\alpha+1)}
\right),
\end{align*}
for all $\epsilon \in (0,\epsilon_0]$. This is the claimed lower bound on the metric entropy.
\end{proof}

\section{Generic approximation results}
\label{sec:4}

We first discuss an abstract formulation of a general ``approximation task''.
Let $\V$ be a Banach space (e.g. a space of operators). In a general non-linear approximation task, we are given for any $n\in \N$ a set $\Sigma_n \subset \V$ over which we aim to approximate an element $f \in \V$, where we will assume that $f$ belongs to a general class $\A \subset \V$ of interest. 
Considering these subsets $\Sigma_n \subset \V$ fixed, and given a sequence $\epsilon_n \to 0$, we will say that $f\in \A$ can be approximated with \emph{convergence rate} $\epsilon_n$, if there exists a constant $M_f > 0$, such that  
\begin{align}
\label{eq:eff}
\inf_{\psi_n \in \Sigma_n} \Vert f - \psi \Vert_\cX \le M_f \epsilon_n, \quad \forall \, n\in \N.
\end{align}

Specifically, we will be most interested in the logarithmic case $\epsilon_n = \log(n)^{-\gamma}$, in the following, with $\Sigma_n$ corresponding to all possible realizations of a fixed bit-encoded neural operator architecture (cp. the proofs of Propositions \ref{prop:uniform-generic} and \ref{prop:Lp-generic}, respectively).

Coming back to the general abstract setting above, and given $M>0$, we introduce a set of ``efficiently approximated'' elements $\EE_M \subset \A$ with bound $M$, i.e. 
\begin{align}
\label{eq:EM}
\EE_M := \set{f \in \A }{\text{inequality \eqref{eq:eff} holds with constant $M_f = M$} }.
\end{align}
And we denote the set of all $f\in \A$ which can be approximated at convergence rate $\epsilon_n$, by elements in $\Sigma_n$, by 
\begin{align} 
\label{eq:E}
\EE = \bigcup_{M>0} \EE_M = \set{f \in \A}{\text{there exists $M_f$ such that \eqref{eq:eff} holds}}.
\end{align}
Our goal is to study \emph{generically} achievable approximation rates $\epsilon_n$, in terms of the complexity of $\A$, as measured by its metric entropy.

The following lemma will be fundamental to our analysis:

\begin{lemma}
\label{lem:basic}
Let $\V$ be a Banach space. Let $\A\subset \V$ be a compact, convex subset. 
Let $\{\Sigma_n\}_{n\in \N}$ be a family of subsets $\Sigma_n\subset \V$, with $|\Sigma_n| \le 2^n$ elements. Fix $M>0$. If $\EE_M\subset \A$ given by \eqref{eq:EM} has non-empty interior in the subspace topology on $\A$, then there exists a constant $\lambda > 0$, independent of $n$, such that the metric entropy satisfies the bound,
\[
\cH(\A; \lambda \epsilon_n)_\V \le n, \quad \forall \, n \in \N.
\]
\end{lemma}

\begin{proof}
At the outset we note that by compactness, we have a uniform upper bound,
\begin{align*}
\sup_{f\in \A} \Vert f \Vert \le C_\A < \infty.
\end{align*}
Upon a simple rescaling, we may wlog assume that $C_\A = 1$, i.e. that $\Vert f \Vert\le 1$ for all $f\in \A$. This will be assumed in the following proof.

Our next goal is to show that, for any $M>0$, the set $\EE_M$ defined by \eqref{eq:EC} has empty interior. For the sake of contradiction, assume that $\EE_M$ does not have empty interior. Then there exists $f_0\in \A$ and $\delta > 0$, such that 
\[
B_\delta(f_0) \subset \EE_M  \subset \bigcup_{\psi_n \in \Sigma_n} \overline{B_{M\epsilon_n}} (\psi_n),
\]
where $B_\delta(f_0) = \set{f\in \A}{\Vert f - f_0 \Vert < \delta}\subset \A$ is an open ball in the subspace topology on $\A$.
Thus, for any $n\in \N$, we obtain the following bound on the covering numbers,
\begin{align}
\label{eq:1}
\cN(B_\delta(f_0); M\epsilon_n) \le \cN(\EE_M;M\epsilon_n) \le |\Sigma_n| \le 2^n.
\end{align}
We next recall that we have wlog assumed $\sup_{f\in \A} \Vert f \Vert \le 1$, and we recall that $\A$ is convex by assumption. In particular, we next show that this implies that 
\[
\left(1 - \frac{\delta}{3}\right)f_0 + \frac{\delta}{3} \A 
\subset B_\delta(f_0).
\]
To see why, let $\delta' = \delta/3$ and fix $f \in \A$ arbitrary. We need to show that $f_{\delta'} := (1-\delta') f_0 + \delta' f \in B_\delta(f_0)$. Since $\A$ is convex, it is clear that $f_{\delta'} \in \A$. In addition, we also have
\[
\Vert f_{\delta'} - f_0 \Vert 
= \Vert (1-\delta') f_0 + \delta' f - f_0 \Vert
= \delta' \Vert f-f_0 \Vert \le 2\delta' = \frac{2\delta}{3} < \delta.
\]
Hence, $f_{\delta'} \in B_\delta(f_0)$ as claimed. The inclusion, $(1-\delta/3) f_0 + (\delta/3) \A \subset B_\delta(f_0)$ now implies,
\begin{align}
\label{eq:2}
\cN(B_\delta(f_0); M\epsilon_n) \ge \cN((\delta/3) \A; M\epsilon_n) = \cN(\A; 3M\epsilon_n / \delta).
\end{align}
Combining \eqref{eq:1} and \eqref{eq:2}, we conclude that 
\[
\cH(\A;3M\epsilon_n/\delta)_\V = 
\log_2 \cN(\A; 3M\epsilon_n / \delta)_\V \le n, \quad \forall \, n \in \N.
\]
We emphasize that $M,\delta > 0$ are independent of $n$ in the above argument. In particular, the claim of the lemma holds with constant $\lambda = 3M/\delta > 0$.
\end{proof}

\begin{proposition}[Exponential scaling]
\label{prop:abs-exp}
Let $\V$ be a Banach space. Let $\A\subset \V$ be a compact, convex subset. Assume that there exist constants $C,c,\gamma>0$ such that,
\begin{align}
\label{eq:Fentropy}
\cH(\A;\epsilon)_\V \ge C \exp\left(c \epsilon^{-1/\gamma} \right), \quad \forall \, \epsilon > 0.
\end{align}
Let $\{\Sigma_n\}_{n\in \N}$ be family of subsets $\Sigma_n\subset \V$ with $|\Sigma_n| \le 2^n$ elements. Then generic elements $f\in \A$ cannot be approximated by elements of $\Sigma_n$ at convergence rate better than $\log(n)^{-\gamma}$; more precisely, for any sequence $\epsilon_n = o(\log(n)^{-\gamma})$, the subset $\RR\subset \A$, consisting of all $f\in \A$, such that
\begin{align}
\label{eq:lower-algebraic}
\inf_{\psi_n \in \Sigma_n} \Vert f - \psi_n \Vert \not = O(\epsilon_n), 
\end{align}
is residual.
\end{proposition}

Before coming to the proof of Proposition \ref{prop:abs-exp}, we note that since $\A\subset \V$ is compact, $\A$ is a complete metric space in the subspace topology. In particular, the following argument, which is based on the Baire category theorem, can be applied to $\A$ (cp. Appendix \ref{app:baire} for a summary). 

\begin{proof}
Let $\RR := \A \setminus \EE$, where $\EE$ is defined by \eqref{eq:E}. Recall that $\EE$ is precisely the set of $f\in \A$ for which there exists $M_f>0$ such that 
\[
\inf_{psi_n \in \Sigma_n} \Vert f - \psi_n \Vert \le M_f \epsilon_n.
\]
In Lemma \ref{lem:basic}, it is shown that if $\EE_M \subset \A$ has non-empty interior then there exists a constant $\lambda > 0$, such that  
\[
\log \cN(\A; \lambda \epsilon_n) \le n, \quad \forall \, n \in \N.
\]
By assumption on $\A$, the left hand side is lower bounded by $C\exp\left(c(\lambda\epsilon_n)^{-1/\gamma}\right)$. Thus, if $\EE_M$ has non-empty interior, then we must have 
\[
C\exp\left((\lambda \epsilon_n)^{-1/\gamma} \right) \le n \quad \Rightarrow \quad \epsilon_n \gtrsim \log(n)^{-\gamma}, \; \text{as }n\to \infty.
\]
But by the assumption that $\epsilon_n = o(\log(n)^{-\gamma})$, this last lower bound cannot hold, asymptotically as $n\to \infty$. Thus, we conclude that $\EE_M \subset \A$ has empty interior for any $M>0$. We furthermore note that $\EE_M$ is closed; indeed, $\EE_M$ in \eqref{eq:EM} is given by,
\begin{align}
\label{eq:EC}
\EE_M = \bigcap_{n=1}^\infty \bigcup_{\psi_n \in \Sigma_n} 
\overline{B_{M \epsilon_n}(\psi_n)},
\end{align}
where we define the closed balls (in the induced topology on $\A$),
\[
\overline{B_{M\epsilon_n}(\psi)}
:= \set{f\in \A}{\Vert f - \psi \Vert \le M \epsilon_n} \subset \A.
\]
Therefore $\EE_M$ can be written as an intersection of a union of closed balls of radius $M \epsilon_n$ centered at elements $\psi \in \Sigma_n$. Note that, since the set $\Sigma_n$ is finite by assumption, the union of these closed balls,
\[
\EE_{M,n} := \bigcup_{\psi_n \in \Sigma_n} 
\overline{B_{M \epsilon_n}(\psi_n)},
\]
is closed for any $n\in \N$, implying that also $\EE_M = \bigcap_{n=1}^\infty \EE_{M,n} \subset \A$ is closed as an intersection of closed sets. 

To conclude the proof, we simply note that $\EE = \bigcup_{M\in \N} \EE_M$ can be written as a countable union, for integer $M\in \N$, of closed subsets with empty interior $\EE_M$. In particular, this implies that $\EE$ is itself meagre by the Baire category theorem. We conclude that the complement $\RR := \A \setminus \EE$, consisting of all $f\in \A$ for which 
\[
\inf_{\psi_n \in \Sigma_n} \Vert f - \psi_n \Vert \not = O(\epsilon_n),
\]
is residual. This completes the proof.
\end{proof}

A similar result can also be derived under the assumption of an algebraic scaling. This may be of relevance for generic function approximation by neural networks, and hence we mention it here, in passing.

\begin{proposition}[Algebraic scaling]
\label{prop:abs-alg}
Let $\V$ be a Banach space. Let $\A\subset \V$ be a compact, convex subset. Assume that there exist constants $C,\gamma>0$ such that,
\begin{align}
\label{eq:Falgebraic}
\log \cN(\A;\epsilon) \ge C \epsilon^{-1/\gamma}, \quad \forall \, \epsilon > 0.
\end{align}
Let $\{\Sigma_n\}_{n\in \N}$ be a family of subsets $\Sigma_n\subset \V$ with $|\Sigma_n| \le 2^n$ elements. Then generic elements $f\in \A$ cannot be approximated by elements of $\Sigma_n$ at convergence rate better than $n^{-\gamma}$; more precisely, for any sequence $\epsilon_n = o(n^{-\gamma})$, the subset $\RR\subset \A$, such that,
\begin{align}
\inf_{\psi_n \in \Sigma_n} \Vert f - \psi_n \Vert \not = O(\epsilon_n), \quad
\forall \, f \in \RR,
\end{align}
is residual.
\end{proposition}

\begin{proof}
Let $\RR := \A \setminus \EE$, where $\EE$ is defined by \eqref{eq:E}. Recall that $\EE$ is precisely the set of $f\in \A$ for which there exists $M_f>0$ such that 
\[
\inf_{psi_n \in \Sigma_n} \Vert f - \psi_n \Vert \le M_f \epsilon_n.
\]
In Lemma \ref{lem:basic}, it is shown that if $\EE_M \subset \A$ has non-empty interior then there exists a constant $\lambda > 0$, such that  
\[
\log \cN(\A; \lambda \epsilon_n) \le n, \quad \forall \, n \in \N.
\]
By assumption on $\A$, the left hand side is lower bounded by $C(\lambda\epsilon_n)^{-1/\gamma}$. Thus, if $\EE_M$ has non-empty interior, then we must have 
\[
C(\lambda \epsilon_n)^{-1/\gamma} \le n \quad \Rightarrow \quad \epsilon_n \gtrsim n^{-\gamma}, \; \text{as }n\to \infty.
\]
By assumption, $\epsilon_n = o(n^{-\gamma})$, this is not the case. Thus, we conclude that $\EE_M \subset \A$ has empty interior for any $M>0$. Thus, arguing as in the proof of Proposition \ref{prop:abs-alg} it follows that $\EE$ is meagre, and hence $\RR = \A\setminus \EE$ is residual.
\end{proof}

\section{Conclusion}
\label{sec:conclusion}

Operator learning is a new paradigm for the data-driven approximation of operators. Popular operator learning frameworks extend and generalize neural networks to this infinite-dimensional setting. While there are numerous papers demonstrating the potential and practical utility of proposed neural operator architectures, our understanding of the precise conditions under which operator learning is practically feasible remains limited. 

This paper makes a contribution to the mathematical underpinnings of this field, by providing an information-theoretic perspective on the curse of parametric complexity (a scaling-limit of the curse of dimensionality) identified in \cite{lanthaler2023operator}. In particular, it is shown that this curse poses a fundamental limitation to operator learning on general spaces of Lipschitz operators. Bit-encoding (storing in memory) any neural operator architecture, which is capable of achieving approximation accuracy $\epsilon$ for general $1$-Lipschitz continuous and real-valued operators, requires a number of bits that is exponential in $\epsilon^{-1}$. It is shown that this is true not only when measuring the approximation error in the sup-norm over compact sets of input functions, but also when measuring the error in the $L^p(\mu)$-norm with respect to a probability measure satisfying certain structural assumptions. The assumptions are met for widely considered $\mu$, including the case of a Gaussian random field with at most algebraically decreasing eigenvalues of the covariance. These results rely on minimax analysis and, in contrast to prior work \cite{lanthaler2023operator}, are independent of the employed activation function in the architecture.

Going beyond such minimax analysis, we furthermore study the approximation of individual Lipschitz operators by a sequence of neural operator architectures. Such a sequence would e.g. be obtained when increasing the width, depth or other hyperparameters at a pre-defined rate as the model is scaled up. In this setting, we address the following question: ``At which rate can the approximation error along such a sequence decrease, as a function of the total number of bit-encoded parameters?'' Using topological arguments based on Baire category, we establish a quantitative relation between the metric entropy of the set of $1$-Lipschitz operators, and the best approximation-rate that can be achieved along such a sequence for generic $1$-Lipschitz operators; as a consequence of the exponential increase in metric $\epsilon$-entropy of the set of $1$-Lipschitz operators, it is shown that achievable approximation rates are at most logarithmic as a function of the required encoding bits.

Finally, this abstract analysis leads to a concrete result on the approximation of generic Lipschitz operators by Fourier neural operator. Our results imply that for generic $1$-Lipschitz operators, and under mild assumptions on the tunable parameters, there cannot exist a sequence of FNO approximations which approximates the underlying operator at a rate that decays faster than logarithmic in the number of \emph{real-valued} parameters. To obtain this result, mild bounds on the growth of the parameters of FNO approximants are assumed; specifically, the size of individual parameter is assumed to be exponentially bounded by the total number of parameters, as the model size is scaled up.

The results of this work should be compared and contrasted with the recent work \cite{schwab2023deep}, which shows the surprising result that there exist (non-standard) neural operator architectures capable of approximating Lipschitz continuous operators to accuracy $\epsilon$, with a number of real-valued tunable parameters $q$ growing only algebraically with $\epsilon^{-1}$. The analysis of the present work indicates that a practical implementation of such architectures on computing hardware, and with parameters encoded by a total of $B$ bits will require $B$ to be exponentially large in $\epsilon^{-1}$. In fact, if each parameter is encoded by $b_1$ bits, then a lower bound of the following form is to be expected:
\[
q b_1 \ge C\exp(c\epsilon^{-\gamma}),
\]
for fixed constants $C,c,\gamma>0$ independent of $\epsilon$. In particular, if $q\lesssim \epsilon^{-\lambda}$ grows at most algebraically, as in the construction \cite{schwab2023deep}, then the number of encoding bits $q_1$ per parameter must necessarily grow exponentially. Thus, the only trade-off that appears possible from an information-theoretic perspective is to reduce the number of parameters $q$ at the expense of the required number of bits per parameter $b_1$, or vice versa. 
In turn, the required number of encoding bits is intimately linked to the stability of the mapping $\theta \mapsto \Phi(\slot;\theta)$ from parameters $\theta$ to the corresponding realization of the neural operator $\Phi(\slot;\theta)$; an exponentially growing number of bits $b_1$ is only required if the parameter-to-realization mapping is either very unstable, e.g. having very large Lipschitz constant, or if the optimal parameters themselves are very large. Here, ``large'' means that either the Lipschitz constant or the $\ell^\infty$-norm of the parameters grows exponentially with $\epsilon^{-1}$.

The results of this work underline the fundamental character of the curse of parametric complexity identified in \cite{lanthaler2023operator} from the point of view of information theory. In addition, it is here shown that this curse persists even when the sup-norm (uniform approximation of the underlying operator) is replaced by an a priori much weaker $L^p$-norm (approximation in expectation). This considerably constrains the generality with which approximation theory for operator learning, guaranteeing efficient approximation by neural operators at algebraic convergence rates, can be developed. A complete or partial characterization of the relevant mathematical properties and structures enabling efficient operator approximation, would be highly desirable. The results presented in this work demonstrate rigorously that one has to go beyond Lipschitz operators to achieve this.

\section*{Acknowledgments}

The author would like to thank Andrew M. Stuart and Nikola B. Kovachki for interesting discussions which have led to this work. This work has been supported by funding from the Swiss National Science Foundation through Postdoc.Mobility grant P500PT-206737.

\bibliographystyle{abbrv}
\bibliography{references}

\appendix

\section{A short summary of Baire category}
\label{app:baire}

In this appendix, we recall the Baire category theorem from general topology. For a more thorough discussion of this result, and its connections to other topological concepts, we refer to the textbook \cite[Chap. 8]{munkres}.

Let $X$ be a topological space. Let $A \subset X$ be a subset. We recall that the interior of $A$ is defined as the union of all open sets of $X$ that are contained in $A$. The set $A$ is said to have empty interior if $A$ contains no open set of $X$ other than the empty set. Equivalently, $A$ is said to have empty interior if the complement of $A$ is dense in $X$.
We then have the following definition \cite[Chap. 8, p. 293]{munkres}:
\begin{definition}
A space $X$ is said to be a \define{Baire space} if the following condition holds:
Given any countable collection $\{A_n\}$ of closed sets of $X$ each of which has empty interior in $X$, their union $\bigcup_{n} A_n$ also has empty interior in $X$.
\end{definition}

This definition can equivalently be stated in terms of open sets \cite[Lemma 48.1]{munkres}:
\begin{lemma}
$X$ is a Baire space if and only if given any countable collection $\{U_n\}$ of open sets in $X$, each of which is dense in $X$, their intersection $\bigcap_n U_n$ is also dense in $X$.
\end{lemma}

The following Baire category theorem \cite[Thm. 48.2]{munkres} exposes many examples of Baire spaces encountered in applications:
\begin{theorem}[Baire category theorem]
If $X$ is a compact Hausdorff space or a complete metric space, then $X$ is a Baire space.
\end{theorem}

\section{Proof of the quantization lemma}
\label{app:quantize}

The goal of this appendix is to prove the FNO quantization lemma \ref{lem:quantize}:

\fnoquant*

\begin{proof}[Proof of Lemma \ref{lem:quantize}]
Let $\Phi_q$ be an output-averaged FNO with at most $q$ tunable parameters. We first note that the depth of $\Phi_q$ can only take the values $L \in \{1,\dots, q\}$. For each possible value of the depth, we now consider the maximally connected output-averaged FNO architecture $\hat{\Phi}_q^{(L)}$ of depth $L$, obtained by setting $\kappa, \dc = q$ in each layer. This maximally connected FNO architecture has at most 
\[
\hat{q}^{(L)} \le 5 (2\kappa)^d L \dc^2 \le 5\cdot 2^d q^{d+3},
\]
tunable parameters. For later reference, we note that 

\vspace{1em}
\textbf{Observation 1:} Any output-averaged averaged FNO $\Phi_q(\slot;\theta)$ with depth $L$ and at most $q$ parameters can be represented by a specific choice of the weights of $\hat{\Phi}_q^{(L)}(\slot;\hat{\theta})$. In fact, this only requires zero-padding $\theta$ to obtain $\hat{\theta}$.
\vspace{1em}

Our main goal is to suitably quantize $\hat{\Phi}_q^{(L)}$, and then define a quantized neural operator architecture $\tilde{\Phi}_{n_q}$ with $n_q$ bits which can represent all quantized $\hat{\Phi}_q^{(L)}$ for $L=1,\dots, q$ by specific setting of its bitwise-encoded parameters. 

It follows from \cite[Proposition D.15]{NLM2024data}, with a minimal extension to allow for $\sigma(0) \ne 0$, that the Lipschitz constant of the mapping,
\[
R_{q}^{(L)}: \; 
\left\{
\begin{aligned}
[-M_q, M_q]^{\hat{q}} &\to C(\cK), \\
\theta &\mapsto \hat{\Phi}_q^{(L)}(\slot;\theta),
\end{aligned}
\right.
\]
and with $[-M_q,M_q]^{\hat{q}}$ metrized by the $\ell^\infty$-norm, 
can be bounded by 
\[
\Lip(R_q^{(L)}) \le (L+2)(2\dc M_q)^{L+2} \left( C + (2\kappa)^{d/2} \right).
\]
Here, $C>0$ is a constant depending only on $d$ and $\cK$. In particular, there exists a (larger) constant $C = C(d,\cK)$, such that 
\[
\Lip(R_q^{(L)}) \le (Cq)^{Cq} = \exp(C q \log(Cq)).
\]
We quantize $\hat{\Phi}_q^{(L)}$ for $\theta\in [-M_q,M_q]^{\hat{q}}$ by subdividing each coordinate direction by equidistant points of separation $\sim \log(q)^{-\gamma} / \exp(Cq \log(Cq))$. Denote the resulting discrete set of points by $\Theta^{(L)}_q\subset \R^{\hat{q}}$. We note that this subdivision requires at most,
\[
O\left( \,
\big\{
M_q\log(q)^\gamma \exp(Cq \log(Cq)) 
\big\}^{\hat{q}}
\,
\right)
\]
many quantization points, which can be encoded by 
\[
O\Big( \,
\hat{q} \, \log\big(\, M_q \log(q)^\gamma \exp(Cq \log(Cq))  \, \big)
\,
\Big)
\]
many bits. Since $\hat{q} = O(q^{d+3})$, $\log(M_q \log(q)^\gamma) = O(q)$ and $\log(\exp(Cq\log(Cq))) = O(q^2)$, it follows that the number of required bits is 
\[
O\left(
q^{d+6}
\right),
\]
i.e. $\log_2 |\Theta^{(L)}_q| = O(q^{d+6})$. The implied constant here is independent of $L$. In the following, we denote $m := d+6$. In particular, we conclude that there exists a constant $C>0$, independent of $q$, such that
\[
\max_{L=1,\dots, q} | \Theta^{(L)}_q | \le Cq^{m}.
\]

We also note that, by construction, for any $\theta \in [-M_q,M_q]^{\hat{q}}$, there exists $\theta'\in \Theta^{(L)}_q$, such that 
\[
\Vert \theta - \theta' \Vert_{\ell^\infty}
\le \frac{\log(q)^{-\gamma}}{\exp(Cq \log(Cq))}.
\]
It follows that for any $\theta \in [-M_q, M_q]^{\hat{q}}$, there exists $\theta'\in \Theta^{(L)}_q$, such that 
\begin{align*}
\Vert \tilde{\Phi}^{(L)}_q(\slot; \theta) - \tilde{\Phi}^{(L)}_q(\slot;\theta') \Vert_{C(\cK)}
&\le
\Lip(R_q^{(L)}) \, \Vert \theta - \theta' \Vert_{\ell^\infty}
\\
&\le \exp(Cq \log(Cq)) \frac{\log(q)^{-\gamma}}{\exp(Cq \log(Cq))} 
\\
&= \log(q)^{-\gamma}.
\end{align*}
Thus, 
\begin{align}
\label{eq:accTL}
\sup_{ \theta \in [-M_q,M_q]^{\hat{q}} } 
\min_{ \theta'\in\Theta^{(L)}_q } 
\Vert \tilde{\Phi}^{(L)}_q(\slot; \theta) - \tilde{\Phi}^{(L)}_q(\slot;\theta') \Vert_{C(\cK)}
\le \log(q)^{-\gamma}.
\end{align}

Since $|\Theta^{(L)}_q|\le Cq^{m}$, any $\theta' \in \Theta^{(L)}_q$ can be identified with a unique bit-string in $\{0,1\}^{\ell_q}$, where $\ell_q = \lceil Cq^{m} \rceil$. Adding an additional number of $O(\log(q))$ bits to encode the possible values of the depth parameter $L\in \{1,\dots, q\}$, we can now define a quantized neural operator $\tilde{\Phi}_{n_q}: L^2(D)\times \{0,1\}^{n_q} \to \R$ encoded by $n_q \sim \log(q) + \ell_q \sim C q^{m}$ bits, in the following way:
Given $[\theta] \in \{0,1\}^{n_q}$, we first read off the length parameter $L$ from the first $\lceil \log_2 q\rceil$ bits. Removing these bits, the remaining $\ell_q$ bits uniquely identify $\theta' \in \Theta_q^{(L)}$, and we set
\begin{align}
\label{eq:TLrep}
\tilde{\Phi}_{n_q}(\slot; [\theta]) := \Phi^{(L)}_q(\slot; \theta').
\end{align}
Thus, $\tilde{\Phi}_{n_q}$ is a neural operator architecture with parameters encoded by $n_q \asymp q^{m}$ bits. By our definition \eqref{eq:TLrep}, any neural operator belonging to the set
\[
\set{
\Phi^{(L)}_q(\slot;\theta')
}{
L \in \{1,\dots, q\}, \; \theta' \in \Theta^{(L)}_q
},
\]
can be represented exactly by suitable choice of $[\theta]\in \{0,1\}^{n_q}$. And thus, by \eqref{eq:accTL}, we have
\begin{align}
\label{eq:accTL2}
\sup_{L=1,\dots, q} 
\sup_{ \theta \in [-M_q,M_q]^{\hat{q}} } 
\min_{ [\theta]\in \{0,1\}^{n_q} } 
\Vert \tilde{\Phi}^{(L)}_q(\slot; \theta) - \tilde{\Phi}_{n_q}(\slot;[\theta]) \Vert_{C(\cK)}
\le \log(q)^{-\gamma}.
\end{align}
We finally note that any neural operator architecture $\Phi_q$ with at most $q$ parameters is represented as $\Phi_q(\slot;\theta) = \hat{\Phi}_q^{(L)}(\slot; \hat{\theta})$ for suitably chosen $\hat{\theta} = \hat{\theta}(\theta)$ (see Observation 1, above). In fact, this only involves zero-padding of the weights $\theta$. In particular, if $\theta \in [-M_q,M_q]^q$, then $\hat{\theta} \in [-M_q,M_q]^{\hat{q}}$. 

From \eqref{eq:accTL2}, it follows that
\begin{align}
\sup_{ \theta \in [-M_q,M_q]^{q} } 
\min_{ [\theta]\in \{0,1\}^{n_q} } 
\Vert \Phi_q(\slot; \theta) - \tilde{\Phi}_{n_q}(\slot;[\theta]) \Vert_{C(\cK)}
\le \log(q)^{-\gamma},
\end{align}
as claimed. This concludes the proof.
\end{proof}

\end{document}